\pdfoutput=1
\documentclass{article}

\usepackage{microtype}
\usepackage{graphicx}
\usepackage{subcaption}

\usepackage{pdfpages}

\usepackage[preprint]{icml2026}

\usepackage{amsmath}
\usepackage{amssymb}
\usepackage{mathtools}
\usepackage{amsthm}

\theoremstyle{plain}
\newtheorem{theorem}{Theorem}[section]
\newtheorem{proposition}[theorem]{Proposition}

\theoremstyle{definition}

\theoremstyle{remark}

\definecolor{myblue}{RGB}{99, 118, 168}
\definecolor{myred}{RGB}{238, 0, 0}

\definecolor{c01}{RGB}{84, 127, 41}
\definecolor{c02}{RGB}{97, 145, 47}
\definecolor{c03}{RGB}{132, 168, 95}
\definecolor{c04}{RGB}{168, 192, 143}
\definecolor{c05}{RGB}{238, 238, 238}
\definecolor{c06}{RGB}{234, 189, 189}
\definecolor{c07}{RGB}{230, 141, 140}
\definecolor{c08}{RGB}{226, 92, 91}
\definecolor{c09}{RGB}{222, 43, 42}

\usepackage[colorlinks=true,
            linkcolor=myred,      
            anchorcolor=myred,    
            citecolor=myblue,  
            ]{hyperref}
\usepackage{url}

\usepackage{wrapfig}
\usepackage{enumitem}
\usepackage{booktabs}
\usepackage{amsthm,amssymb}
\usepackage{multirow}
\usepackage{algorithmic}
\usepackage{xcolor,color}
\usepackage{soul}
\usepackage{tcolorbox}
\usepackage{colortbl}
\usepackage{tikz}
\usepackage{minted}
\usetikzlibrary{calc}
\usepackage{tabularx}
\usepackage{ragged2e}
\newcolumntype{L}{>{\RaggedRight\arraybackslash}X}

\newcommand\eg{\emph{e.g}.} 

\newcommand\ie{\emph{i.e}.}

\usepackage{accents}

\newcommand{\W}{\mathbf{W}}
\newcommand{\A}{\mathbf{A}}

\newcommand{\K}{\mathbf{K}}
\newcommand{\V}{\mathbf{V}}
\newcommand{\Q}{\mathbf{Q}}

\newcommand{\BS}{\mathbf{S}}

\newcommand{\BR}{\mathbf{R}}
\usepackage{bm}

\usepackage{xcolor}

\definecolor{reasonbg}{RGB}{230,230,255}  
\definecolor{verifybg}{RGB}{255,245,205}  

\newcommand{\algshade}[2]{%
  \colorbox{#1}{\parbox{\dimexpr\linewidth-2\fboxsep\relax}{\strut #2\strut}}%
}

\usepackage[textsize=tiny]{todonotes}

\icmltitlerunning{One-Token Verification for Reasoning Correctness Estimation}

\begin{document}

\twocolumn[
  \icmltitle{One-Token Verification for Reasoning Correctness Estimation}



  \icmlsetsymbol{equal}{*}

\begin{icmlauthorlist}
\icmlauthor{Zhan Zhuang}{sustech,cityu}
\icmlauthor{Xiequn Wang}{sustech}
\icmlauthor{Zebin Chen}{sustech}
\icmlauthor{Feiyang Ye}{sustech}
\icmlauthor{Ying Wei}{zju}
\icmlauthor{Kede Ma}{cityu}
\icmlauthor{Yu Zhang}{sustech}
\end{icmlauthorlist}

\icmlaffiliation{sustech}{Southern University of Science and Technology, Shenzhen, China}
\icmlaffiliation{cityu}{City University of Hong Kong, Hong Kong SAR}
\icmlaffiliation{zju}{Zhejiang University, Hangzhou, China}

\icmlcorrespondingauthor{Ying Wei}{ying.wei@zju.edu.cn}
\icmlcorrespondingauthor{Kede Ma}{kede.ma@cityu.edu.hk}
\icmlcorrespondingauthor{Yu Zhang}{yu.zhang.ust@gmail.com}

  \icmlkeywords{Machine Learning, ICML}

  \vskip 0.3in
]



\printAffiliationsAndNotice{}  

\begin{abstract}
Recent breakthroughs in large language models (LLMs) have led to notable successes in complex reasoning tasks, such as mathematical problem solving. A common strategy for improving performance is \textit{parallel thinking}, in which multiple reasoning traces are generated and the final prediction is made using aggregation schemes like majority voting or best-of-$N$ decoding. However, two key challenges persist. First, multi-sample decoding incurs substantial inference latency, especially for long-form outputs. Second, effective mechanisms for reliably assessing the correctness of individual reasoning traces are still limited. To address these challenges, we introduce One-Token Verification~(OTV), a computational method that estimates reasoning correctness in a single forward pass during generation. OTV is activated by a learnable token and integrated into the LLM via low-rank adaptation to probe internal reasoning signals through the key-value cache, supporting token-level correctness estimation at any stage of generation without disrupting primary reasoning. Experiments on mathematical reasoning benchmarks demonstrate that OTV consistently surpasses existing verifiers. Additionally, OTV reduces token usage by up to $90\%$ through correctness-guided early termination, prioritizing shorter, more reliable solutions. 
\end{abstract}

\section{Introduction}
Large language models (LLMs) such as OpenAI o1~\citep{jaech2024openai}, DeepSeek-R1~\citep{guo2025deepseek}, and the Qwen3 series~\citep{yang2025qwen3} have recently demonstrated strong multi-step reasoning capabilities on challenging tasks like mathematical problem solving. These advances are largely driven by training pipelines that combine supervised fine-tuning from human-supplied long chains of thought~\citep{wei2022chain,suzgun2023challenging} and reinforcement learning from outcome- or process-level feedback~\citep{ouyang2022training,shao2024deepseekmath,team2025kimi}. In parallel, \textit{test-time scaling} has emerged as a complementary paradigm that improves accuracy by allocating additional computation during inference~\citep{brown2024large,venktesh2025trust}. A simple yet effective instance is \textit{parallel thinking}, where the model generates multiple candidate solution traces and aggregates them into a final answer.

\begin{figure}[t]
    \centering
    \includegraphics[width=\linewidth]{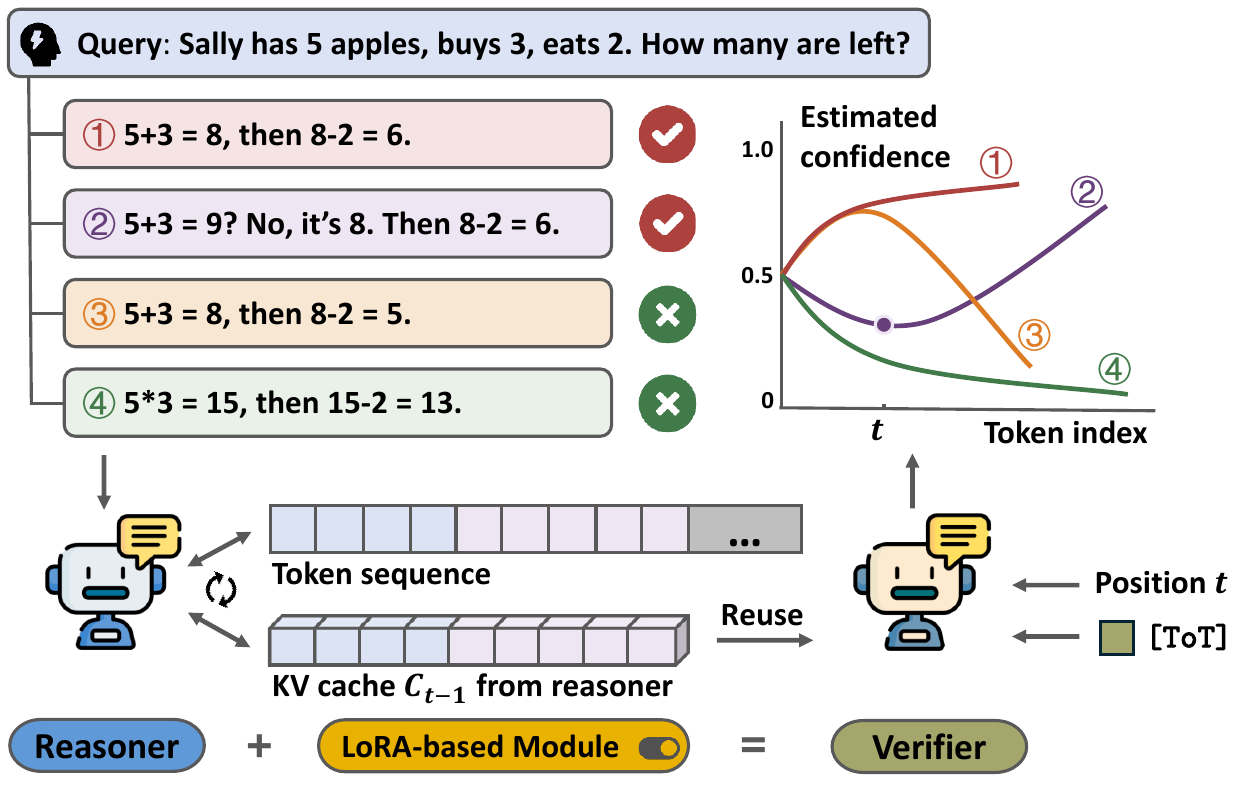}
    \caption{Conceptual illustration of the proposed OTV. By reusing the KV cache and activating a LoRA-based verifier via a special token \texttt{[ToT]}, OTV reliably estimates the correctness of reasoning traces in a single forward pass.}
    \vspace{-0.2in}
    \label{fig:illustration}
\end{figure}

A central challenge in parallel thinking is to reliably assess the correctness of individual reasoning traces. Existing approaches can be broadly grouped into \textit{internal} and \textit{external} verification. Internal methods rely on the model's own token-level uncertainty or calibration\footnote{Calibration refers to the agreement between predicted probabilities and empirical outcome frequencies—for example, among predictions with probability $0.7$, roughly $70\%$ should be correct.} to derive trace-level confidence\footnote{Here, confidence denotes the model’s predicted probability that its generated answer is correct.} scores, which are then used to perform majority voting over diverse traces~\citep{wang2022self, kang2025scalable,fu2025deep,zhang2025reasoning,huang2025efficient}. These techniques are \textit{model-native}, requiring no auxiliary models, but they often suffer from miscalibration and struggle to separate correct from incorrect reasoning traces, particularly for long-form solutions~\citep{huang2023large,xiongcan}. External methods instead train dedicated verifiers~\citep{cobbe2021training,hosseini2024v,zhang2024generative,lightman2023let,wang2024interpretable,yang2024qwen25mathtechnicalreportmathematical,zhao2025majority} to score intermediate steps or final answers. Even though these auxiliary models can deliver richer feedback, they treat the base LLM as a black box and often add substantial inference overhead, while also risking domain mismatch. As a result, their achieved accuracy still remains well below the upper bound suggested by the Pass@$k$ metric~\citep{chen2021evaluating}. Moreover, existing approaches typically defer decision making until complete reasoning traces are generated, because their scoring rules are defined over final answers or full trajectories. When early termination is enabled, performance may degrade noticeably if intermediate signals are unreliable. Consequently, decoding cost, as exacerbated by ``System-2''-style overthinking~\citep{chen2024not}, raises significant efficiency concerns.

These limitations call for \textit{a deeper form of internal verification} that goes beyond surface-level, logit-based heuristics. Such verification should 1) explicitly leverage the internal states of the ongoing reasoning process~\citep{burns2022discovering,azaria2023internal,zhang2025reasoning,li2025diffusion}, 2) provide fine-grained, token-level estimations of reasoning correctness to unlock practical early termination~\citep{fu2025deep,zhang2025reasoning,lee2025tokensupervised}, and 3) be inexpensive to invoke during generation.

To this end, we introduce One-Token Verification~(OTV), a computational method that augments a reasoning LLM with a LoRA-based verifier~\citep{hu2021lora}, operating directly on the model's key-value (KV) cache (see Figure~\ref{fig:illustration}). Concretely, OTV introduces a special \textit{token of truth} (\texttt{[ToT]}). When inserted at inference time, this token leverages LoRA-enhanced cross attention to access the cached KV states accumulated during the ongoing reasoning process. A small regression head maps the last-layer hidden state of \texttt{[ToT]} to a scalar estimate of reasoning correctness, providing token-level verification of the partial trace. Crucially, the LoRA updates are gated so that the base LLM behaves identically to the original reasoner in its default mode, and adopts a verifier role only conditioned on the presence of \texttt{[ToT]} in the input. OTV thus reuses the LLM's full internal computation, and incurs only a single forward pass per verification query.

Training OTV relies solely on cheap pseudo-scores assigned to every token in a reasoning trace, derived from its final correctness label. We further show that OTV training admits parallelization: by inserting \texttt{[ToT]} simultaneously at all candidate positions and reusing a cached prefix, OTV computes correctness scores for the entire trace in a single forward pass. This preserves the token-level parallelism of standard Transformer training. We evaluate OTV on multiple reasoning LLMs (\ie, Qwen3-4B-Instruct, Qwen3-8B, and DAPO-Qwen-32B) and a suite of math benchmarks. Across diverse parallel thinking strategies and several efficient variants, OTV consistently outperforms prior internal and external verifiers, substantially narrowing the gap toward Pass@$k$ while reducing token usage by up to $90\%$ via confidence-guided early termination. Beyond post-trained reasoning LLMs, we show that calibrating OTV on a pretrained base model also yields sizable gains in raw mathematical reasoning accuracy and robustness.

\section{Related Work}

\paragraph{Parallel thinking} 
Test-time scaling has become a 
central paradigm for strengthening LLM reasoning, and it is typically pursued along three axes:
enlarging the effective input context (\eg, retrieval-augmented generation~\citep{lewis2020retrieval} and tool use~\citep{schick2023toolformer}), increasing sequential reasoning depth~\citep{jaech2024openai,muennighoff2025s1}, and widening generation through parallel sampling~\citep{comanici2025gemini,wen2025parathinker,yang2025multiverse,hsu2025group,zheng2025parallel}. Among these, parallel thinking stands out as one of the most widely adopted approaches.  Representative instances include majority voting~\citep{wang2022self}, Best-of-$N$ decoding~\citep{stiennon2020learning}, and tree- or graph-structured search~\citep{yao2023tree,zhang2024accessing}. These methods explore and evaluate multiple (partial) solutions before committing to a final answer, which generally achieve a higher performance ceiling than single-trace decoding~\citep{ghosal2025does}, but at the cost of substantial token overhead. To control this budget, recent work investigates trajectory pruning~\citep{lee2025tokensupervised,wang2025sampling,fu2025deep,huang2025efficient} to terminate low-promise traces early.

\paragraph{Reasoning correctness estimation}
Assessing the correctness of individual reasoning traces~\citep{lee2025evaluating} is critical for reliable early termination of LLMs in parallel thinking. Existing approaches can be grouped into internal and external verification. Internal methods exploit the model's own representations~\citep{lin2022teaching,fadeeva2024fact}. For example, self-consistency~\citep{wang2022self} aggregates diverse traces via majority voting, while self-certainty~\citep{kang2025scalable} derives confidence from output distributions. DeepConf~\citep{fu2025deep} prunes low-confidence traces, and self-calibration~\citep{huang2025efficient} distills confidence scores into a single-pass estimator. Probing-based methods further train prediction heads on last-layer hidden states~\citep{zhang2025reasoning, lee2025tokensupervised}, showing that models internally encode correctness signals that can support calibrated early exits.

External methods instead introduce auxiliary verifiers that operate on the generated text. Outcome reward models assess only the final answer~\citep{cobbe2021training,yu2023ovm,chen2024tree,liu2024acemath,lu2024autopsv,zhang2025rvllm}, whereas process reward models score intermediate steps and aggregate them into an overall correctness estimate~\citep{uesato2022solving,lightman2023let,wang2023math,prmlessons}. 
Subsequent work extends verification to broader domains~\citep{zeng2025versaprm} and richer criteria~\citep{golovneva2022roscoe,wang2024interpretable}, or reframes verification as a reasoning-induced prediction task~\citep{ankner2024critique,zhang2024generative}. Other lines develop critic-style feedback models~\citep{zheng2023judging,zheng2024processbench,ye2025uncertainty} or ensembles over multiple verifiers and solutions~\citep{lifshitz2025multi,zhong2025solvedetectverifyinferencetimescalingflexible,zhao2025majority} to improve reliability. 

Although external verifiers often provide stronger feedback, they ignore the base LLM's internal dynamics and introduce extra inference cost, with potential domain mismatch. In contrast, internal methods are model-native and efficient, but remain susceptible to miscalibration, especially for long-form solutions.

OTV lies between these two extremes. By attaching a LoRA-based module that is activated only in verification mode, OTV allows the same reasoning LLM to also operate as a verifier. It provides token-level correctness estimates with negligible additional computation, thereby combining the model-specificity and efficiency of internal approaches with explicit scoring capabilities of external verifiers.

\section{Proposed Method: OTV}
In this section, we introduce OTV, a computational method for reasoning correctness estimation. OTV integrates a reasoning LLM with 1) a LoRA-based module that is activated only during verification
and 2) a special verification token that probes the model's KV cache. We now describe the four components: the LoRA module design (Sec.~\ref{sec3_1}), KV cache probing (Sec.~\ref{sec_2}), pseudo-labeling (Sec.~\ref{sec_3}), and parallelized implementation (Sec.~\ref{sec_4}).

\subsection{LoRA-based Verification}\label{sec3_1}
LoRA~\citep{hu2021lora} is a widely used parameter-efficient fine-tuning method that augments pretrained weight matrices with low-rank updates. Given a weight matrix $\mathbf{W} \in \mathbb{R}^{d_{\mathrm{out}} \times d_{\mathrm{in}}}$, LoRA introduces two trainable matrices $\BR \in \mathbb{R}^{d_{\mathrm{out}} \times r}$  and $\BS \in \mathbb{R}^{r \times d_{\mathrm{in}}}$ with rank $r \ll\min \{d_{\mathrm{in}}, d_{\mathrm{out}}\}$, and defines the updated weight as $\W_\mathrm{LoRA} =  \W + \BR\BS$. LoRA only trains the adapters, which substantially reduces the trainable-parameter count and optimizer-state memory. 

In OTV, we attach a LoRA-based verification module to selected layers of the base LLM. To preserve the original reasoning ability, we adopt a gating mechanism~\citep{samragh2025your}, which adds the LoRA pathway in parallel to each linear layer and activates it only in verification mode. For an input $\bm x_t \in \mathbb{R}^{d_{\mathrm{in}}\times 1}$ at position $t$, the corresponding output is
\begin{equation}\label{eq:mm}
    \bm z_t = (\W+ m_t\,\BR\BS)\,\bm x_t,
\end{equation}
where $m_t \in \{0,1\}$ is a binary gate. When $m_t=0$, the LLM behaves identically to the original reasoner; when $m_t=1$, the LoRA update is applied, and the model assumes the verifier role.
The gating provides a clean separation between reasoning and verification, and minimizes the risk that fine-tuning degrades the base model's reasoning capabilities.

\subsection{KV Cache-based Internal Representation Probing}\label{sec_2}
During autoregressive decoding, Transformer-based LLMs maintain a KV cache at layer $l$ and position $t$, denoted as $\mathcal{C}_{t}^{(l)} = \{\mathbf{K}_{t}^{(l)}, \mathbf{V}_{t}^{(l)}\}$ where $\mathbf{K}_{t}^{(l)} = [\bm k_1^{(l)}, \ldots, \bm k_{t}^{(l)}]$,  $\mathbf{V}_{t}^{(l)} = [\bm v_1^{(l)}, \ldots, \bm v_{t}^{(l)}]$, and $\bm k_{t}^{(l)}, \bm v_{t}^{(l)} \in \mathbb{R}^{D\times 1}$ are $D$-dimensional key and value vectors, with $\mathcal{C}_{0}^{(l)} = \emptyset$. 
For an $L$-layer LLM, the forward pass at position $t+1$ is functionally determined by the current input $\bm x_{t+1}$ and the accumulated KV cache $\mathcal{C}_{t}=\bigcup_{l=1}^{L}\mathcal{C}^{(l)}_{t}$. Compared with the last-layer hidden states~\citep{zhang2025reasoning, lee2025tokensupervised}, which represent a lossy summary of the preceding context and discard its token-wise and layer-wise structure, the KV cache serves as a \textit{sufficient statistic} of the prefix for the underlying LLM, allowing richer, task-specific pooling over the entire reasoning trajectory.

OTV is designed to explicitly exploit this structure. We introduce a special \emph{token of truth} (\texttt{[ToT]}), which is inserted only in verification mode (\ie, when $m_{t+1}=1$ in Eq.~(\ref{eq:mm})) at an arbitrary position ${t+1}$ to estimate the correctness of the partial trace up to token $t$. Instead of recomputing the prefix, the model reuses the KV cache $\mathcal{C}_{t}$ at every layer and performs a single forward pass augmented with the LoRA-based verification module, described in Sec.~\ref{sec3_1}. 

Let $\bm x_{t+1}^{(l)}$ be the input to layer $l$, with $\bm x_{t+1}^{(0)}$ representing the initial embedding of \texttt{[ToT]}. The LoRA-augmented query, key, and value vectors are given by
\begin{equation}
\begin{cases}
\tilde{\bm q}_{t+1}^{(l)} = \big(\W_{q}^{(l)} + \BR_q^{(l)} \BS_q^{(l)}\big)\, \bm x_{t+1}^{(l)}, \\
\tilde{\bm k}_{t+1}^{(l)} = \big(\W_{k}^{(l)} + \BR_k^{(l)} \BS_k^{(l)}\big)\, \bm x_{t+1}^{(l)}, \\
\tilde{\bm v}_{t+1}^{(l)} = \big(\W_{v}^{(l)} + \BR_v^{(l)} \BS_v^{(l)}\big)\, \bm x_{t+1}^{(l)}.
\end{cases}
\end{equation}
where $\W_{\ast}^{(l)}$, $\ast \in \{q,k,v\}$, denotes the frozen pretrained weights at layer $l$ and $\BR_\ast^{(l)},\BS_\ast^{(l)}$ are the corresponding trainable LoRA matrices. Using the cached prefix $\mathcal{C}_{t}$ and the updated vectors $\tilde {\bm q}_{t+1}^{(l)},\tilde {\bm k}_{t+1}^{(l)},\tilde {\bm v}_{t+1}^{(l)}$, the attention output
 for \texttt{[ToT]} at layer $l$ is
\begin{equation}\label{eq:one_token_compute}
\resizebox{0.9\hsize}{!}{$ 
    \tilde{\bm h}^{(l)}_{t+1} = \left[\,\V^{(l)}_{t},\; \tilde {\bm v}^{(l)}_{t+1}\,\right] \mathrm{softmax}\!\left(\frac{ \big[\,\K^{(l)}_{t},\; \tilde {\bm k}^{(l)}_{t+1}\,\big]^\intercal \tilde {\bm q}_{t+1}^{(l)}}{\sqrt{D}}\right),
$}
\end{equation}
where $\mathrm{softmax}(\cdot)$ is applied column-wise.
After \texttt{[ToT]} is propagated through all $L$ layers, we obtain its final hidden state $\tilde{\bm h}^{(L)}_{t+1} \in \mathbb{R}^{D\times 1}$. OTV applies a three-layer perceptron $g(\cdot)$ to map $\tilde{\bm h}^{(L)}_{t+1}$ to a scalar prediction:
\begin{equation}\label{eq:regression}
    \hat{c}_{t} = g\left(\tilde{\bm h}^{(L)}_{t+1}\right) \in [0,1],
\end{equation}
which estimates the likelihood (\ie, confidence) that the reasoning trajectory is correct up to position $t$. Predicting a continuous score rather than decoding a token from the vocabulary avoids entangling verification with the model’s linguistic prior over specific tokens (\eg, ``correct'' and ``wrong'') and yields a signal that is easier to calibrate and threshold for routing. 

\subsection{Token-level Pseudo-confidence Labeling}\label{sec_3}
OTV requires token-level supervision to train the verifier, but collecting such process-supervision signals~\citep{lightman2023let} or generating search-based rollouts is costly~\citep{wang2023math,luo2024improve,zhang2024rest,feng2024step,setlur2024rewarding,guan2025rstar}. Instead, we derive pseudo-confidence targets solely from outcome-level supervision. Concretely, given a training dataset, we sample a reasoning trace $\bm x_{1:T} = [\bm x_1, \ldots, \bm x_{T}]$ and assign a scalar target $c_t \in [0,1]$ at each token position $t$. We interpret $c_t=0$ as \textit{confidently incorrect},  $c_t=1$ as \textit{confidently correct}, and $c_t=0.5$ as \textit{maximally uncertain}. Let $y \in \{0, 1\}$ indicate the final correctness of the full trace.  In the default setting, we initialize $c_0 = 0.5$ as a neutral prior and progressively bias confidence toward $y$ as evidence accrues. We implement this using a \emph{Linear ramp} that enforces a monotone, linear interpolation from uncertainty to the final outcome:
\begin{equation}\label{eq:psl}
    c_t =\mathrm{linear}(t)=0.5 + (y - 0.5)\frac{t}{T}.
\end{equation}
After assigning token-level pseudo-confidence labels, we train the LoRA-based verification module and the regression head by minimizing the mean squared error (MSE) over response tokens:
\begin{equation}
    \ell = \frac{1}{T}\sum_{t=1}^T(c_t - \hat{c}_t)^2.
\end{equation}
Although we adopt the \textit{linear ramp} as our default pseudo-confidence prior, we also explore several alternative labeling rules (\eg, constant, sigmoid, noise-perturbed, and stepwise ramps). In addition, we present detailed ablation studies to evaluate performance and provide a theoretical analysis of our pseudo-confidence labeling rules in Sec.~\ref{sec:ablation} and
Appendix~\ref{app:theory}, respectively.

\subsection{Parallelization}\label{sec_4}
A key advantage of Transformers~\citep{vaswani2017attention} is their ability to train with token-level parallelism. OTV preserves this property: although the confidence score at position $t$ is obtained by one-token verification (in Eq.~(\ref{eq:one_token_compute})), we can compute scores for all positions in a trace with a single forward pass by probing all prefixes in parallel.

Consider a reasoning trace of length $T$. For each prefix ending at $t\in\{1,\ldots, T\}$, we want a verifier query that attends only to the cached KV $\mathcal{C}_{t}$ associated with that prefix (\ie, the first $t$ tokens). To this end, we construct a \textit{probe} sequence of $T+1$ truth tokens, $\texttt{[ToT]}_{1:T+1}$. In verification mode, we collect their LoRA-augmented projections at layer $l$ as $\widetilde {\mathbf{Q}}^{(l)}_{T+1}, \widetilde {\mathbf{K}}^{(l)}_{T+1}, \widetilde {\mathbf{V}}^{(l)}_{T+1} \in \mathbb{R}^{D\times (T+1)}$. To enforce that probe token $t+1$ only ``sees'' the prefix up to $t$, we use a triangular mask $\mathbf{M} \in \mathbb{R}^{T\times (T+1)}$ defined as:
\begin{equation}
M_{i,j}=
\begin{cases}
0, & i < j,\\
-\infty, & i \ge j,
\end{cases}
\end{equation}
where $i \in \{1,\ldots,T\}$ and $j \in \{1,\ldots,T+1\}$.
At layer $l$, the cache-to-probe attention logits are 
\begin{equation}
\A^{(l)}_{\mathrm{ctp}}
=
\left(\K^{(l)}_{T}\right)^{\intercal}\,\widetilde{\Q}^{(l)}_{T+1}
+ \mathbf{M}.
\end{equation}
To match the single-token computation (in Eq.~(\ref{eq:one_token_compute})), each probe token shall attend only to itself within the set of probe tokens. Accordingly, we compute the per-probe self-logits as the diagonal of the probe-to-probe dot-product matrix:
\begin{equation}
\bm a^{(l)}_{\mathrm{self}}
=
\operatorname{diag}\!\left(\left(\widetilde{\K}^{(l)}_{T+1}\right)^{\intercal}
\widetilde{\Q}^{(l)}_{T+1}\right)
\in \mathbb{R}^{(T+1)\times 1},
\end{equation}
where $\operatorname{diag}(\cdot)$ extracts the diagonal entries of a square matrix and returns them as a column vector.
We then append these self-logits as an additional row and apply a column-wise softmax to obtain the attention probabilities:
\begin{equation}
\begin{aligned}
\A^{(l)}_{T+1}
&=
\begin{bmatrix}
\A^{(l)}_{\mathrm{ctp}}\\
\left(\bm a^{(l)}_{\mathrm{self}}\right)^{\intercal}
\end{bmatrix}
\in \mathbb{R}^{(T+1)\times (T+1)},
\\
\mathbf{P}^{(l)}_{T+1}&=\begin{bmatrix}
\mathbf{P}^{(l)}_{\mathrm{ctp}}\\
\left(\bm p^{(l)}_{\mathrm{self}}\right)^{\intercal}
\end{bmatrix}
=
\operatorname{softmax}\!\left(\frac{\A^{(l)}_{T+1}}{\sqrt{D}}\right).
\end{aligned}
\end{equation}
The resulting attention outputs for all probes at layer $l$ are 
\begin{equation}
\widetilde{\mathbf{H}}^{(l)}_{T+1}
=
\V^{(l)}_{T}\, \mathbf{P}^{(l)}_{\mathrm{ctp}}
+
\widetilde{\V}^{(l)}_{T+1}\,
\operatorname{diag}\!\left(\bm p^{(l)}_{\mathrm{self}}\right)
\in \mathbb{R}^{D\times (T+1)}.
\end{equation}
Here, $\operatorname{diag}(\cdot)$ converts a column vector into a square diagonal matrix. Finally, stacking these computations across all $L$ layers yields $T+1$ final \texttt{[ToT]} hidden states, each of which is mapped to a scalar confidence (in Eq.~(\ref{eq:regression})). This construction is mathematically equivalent to running $T+1$ independent single-token verification queries. However, it consolidates these queries in a single forward pass by reusing the same cached prefix. We apply the MSE loss only at probe positions corresponding to response tokens, excluding those in the prompt and query segment. Algorithmic descriptions of the training and inference procedures are given in Appendix~\ref{appendix:algorithm}.

\section{Experiments}\label{sec:experiments}

In this section, we test OTV for parallel thinking across multiple open-source reasoning LLMs in both standard and advanced mathematical reasoning settings. We consider common aggregation schemes, including self-consistency, weighted self-consistency, and best-of-$N$, along with early-termination variants. We report accuracy and efficiency metrics (\ie, token cost, and verification overhead) and provide a qualitative analysis of OTV confidence trajectories over the course of generation.

\subsection{Experimental Setups}\label{sec:setup}
\paragraph{Models and datasets} We evaluate OTV on three representative open-source reasoning LLMs: \textit{Qwen3-4B-Instruct}, \textit{Qwen3-8B}~\citep{yang2025qwen3}, and \textit{DAPO-Qwen-32B}~\citep{yu2025dapo}. In all cases, the base LLM is kept frozen and used as the reasoner. To cover a range of reasoning difficulty, we consider two experimental regimes. For \textit{standard mathematical reasoning}, we train OTV on MetaMathQA~\citep{yu2023metamath} and test on GSM8K~\citep{cobbe2021training}, a widely adopted benchmark of grade-school math problems. For \textit{advanced mathematical reasoning}, we train on DAPO17K~\citep{yu2025dapo}, a more challenging dataset that provides only final answers (without human-written reasoning traces), and evaluate on the AIME24~\citep{AIME2024} and AIME25~\citep{AIME2025} benchmarks, which comprise contest-style problems from the American Invitational Mathematics Examination. Our main experiments focus on this advanced regime, as it rigorously tests long-horizon, multi-step logical reasoning.

\paragraph{Training protocol} For each training instance, we sample the base LLM $8$ times to obtain diverse reasoning traces, and derive token-level pseudo-confidence targets from the final correctness label to supervise OTV. At inference time, for each test question, we first generate a pool of $256$ complete reasoning traces, setting the temperature parameter to $1.0$, and repeatedly sample subsets from this pool for downstream aggregation. Unless otherwise stated, we fine-tune LoRA modules (applied to query, key, and value projections) together with the regression head using \texttt{LlamaFactory}~\citep{zheng2024llamafactory} for $3$ epochs, with a learning rate of $10^{-4}$ and a batch size of $128$. 

\paragraph{Competing methods}  
We categorize competing methods into internal and external verifiers. Internal methods are model-native, including DeepConf~\citep{fu2025deep}, a training-free confidence estimator derived from token-level log-likelihood, and GenRM~\citep{zhang2024generative}, a reasoning-induced verifier that prompts the model itself with ``Is the answer correct?'' and uses the predicted ``Yes/No'' probability as a confidence score. External methods adopt off-the-shelf reward models, including AceMath-RM-7B~\citep{liu2024acemath}, an outcome-level reward model that scores the full solution, and several process reward models---VersaPRM~\citep{zeng2025versaprm}, Math-Shepherd-Mistral-7B~\citep{wang2023math}, and Qwen2.5-Math PRM variants~\citep{prmlessons}---that assign step-level rewards along the reasoning trace. Following prior practice, we take the reward predicted at the last token as the overall verifier score.

\begin{table*}[t]
\centering
\renewcommand{\arraystretch}{1.0}
\caption{Weighted majority-voting accuracy (\%) on AIME. For process reward models, we use the final token score as the confidence for each reasoning trace. Following~\citet{fu2025deep}, we use method-specific aggregation windows: confidence is computed as the mean score over the last tokens of each trace ($100$ tokens for OTV;  $2,048$ tokens for DeepConf). We then discard the bottom $50\%$ of traces by confidence before performing weighted majority voting. Each run samples $128$ reasoning traces, and all results are averaged over $64$ runs. We also report Pass@$128$ as an oracle upper bound for parallel thinking. Within each column, the best result is highlighted in bold, and the second best is underlined.}
\vspace{-0.2cm}
\begin{small}
\begin{sc}
\resizebox{\textwidth}{!}{
\begin{tabular*}{\textwidth}{@{\extracolsep{\fill}}clcccccc}
\toprule
 & & \multicolumn{2}{c}{{Qwen3-4B}} & \multicolumn{2}{c}{{Qwen3-8B}} &  
 \multicolumn{2}{c}{{DAPO-Qwen-32B}}\\
\cmidrule(lr){3-4}
\cmidrule(lr){5-6}
\cmidrule(lr){7-8}
 & & AIME24 & AIME25 & AIME24 & AIME25 & AIME24 & AIME25   \\
\midrule 
& Pass@128 & 91.46
& 83.13
& 69.58
& 56.98 
& 83.75
& 68.75
\\
\midrule
\multirow{4}{*}{\rotatebox{90}{\textit{Internal}}} & Pass@1   & 60.29 & 46.67 & 26.29 & 19.32
& 51.77 & 36.42\\
 & Maj@128 & 75.42\small{$\pm$1.61} & 66.46\small{$\pm$1.16} & 44.22\small{$\pm$2.11} &28.44\small{$\pm$2.06}
 & 66.72\small{$\pm$0.93} & 41.77\small{$\pm$1.76}\\
 & DeepConf & 77.76\small{$\pm$2.43} & \underline{66.77}\small{$\pm$1.18} & \underline{45.73}\small{$\pm$1.50} & 31.87\small{$\pm$2.35} & 
 66.77\small{$\pm$1.17} &44.79\small{$\pm$1.65}\\
 & GenRM & \underline{79.11}\small{$\pm$2.06} & 66.72\small{$\pm$1.82}& 44.38\small{$\pm$2.27}  & \underline{32.66}\small{$\pm$2.30}
 & 66.72\small{$\pm$0.59} & 42.55\small{$\pm$1.66} \\
\midrule
 \multirow{5}{*}{\rotatebox{90}{\textit{External}}} & AceMath-RM-7B & 67.66\small{$\pm$1.74} & 60.00\small{$\pm$2.36} & 44.90\small{$\pm$2.70} & 26.72\small{$\pm$2.18}
 & 61.15\small{$\pm$2.22} & 42.76\small{$\pm$2.00}\\
 & VersaPRM-8B & 75.52\small{$\pm$1.69} & 66.72\small{$\pm$1.10} & 44.48\small{$\pm$2.30} & 29.11\small{$\pm$2.22} 
 & 66.98\small{$\pm$0.97} & 43.18\small{$\pm$1.81}\\
 & Math-Shepherd-7B & 74.27\small{$\pm$1.50} & 66.56\small{$\pm$1.02} & 44.22\small{$\pm$1.79} & 26.04\small{$\pm$1.94}
 & 65.73\small{$\pm$2.24} & 43.65\small{$\pm$1.93}\\
 & Qwen2.5-PRM800K-7B &  75.16\small{$\pm$1.66} & 66.35\small{$\pm$1.14} & 44.22\small{$\pm$2.30} & 29.32\small{$\pm$2.57}
 & 67.03\small{$\pm$1.46} & \underline{48.18}\small{$\pm$3.48}\\
 & Qwen2.5-PRM-7B & 78.18\small{$\pm$1.85} & 63.49\small{$\pm$2.39} & 45.10\small{$\pm$2.35} & 29.48\small{$\pm$2.63}
 & \underline{67.45}\small{$\pm$2.26} & 46.77\small{$\pm$3.06}\\
\midrule
 & OTV (Ours) & \textbf{83.33}\small{$\pm$1.57} & \textbf{69.32}\small{$\pm$1.46} & \textbf{46.56}\small{$\pm$1.25} & \textbf{33.85}\small{$\pm$1.69}
 & \textbf{70.83}\small{$\pm$1.86} & \textbf{49.58}\small{$\pm$1.10}\\
\bottomrule
\end{tabular*}
}
\end{sc}
\end{small}
\label{tab:majority_voting}
\end{table*}

\paragraph{Inference-time aggregation strategies}  
We evaluate three families of parallel-thinking decoders. 
\begin{itemize}[leftmargin=15pt]
    \item \textbf{Self-consistency} (\ie, majority voting) and its weighted variant~\citep{wang2022self}, where each trace's final answer is weighted by its estimated confidence score.
    \item \textbf{Best-of-$N$}, which selects the highest-scoring trace among $N$ candidates.
    \item \textbf{Efficient Best-of-$N$ variants},
    which reduce computation via pruning or early termination:
    \begin{itemize}[leftmargin=1.5em]
        \item \textit{Drop@$10$}: Periodic score-based pruning.  For every $10$ generated tokens, it drops the currently lowest-scoring trace and continues decoding until only one trace remains. 
        \item \textit{Stop@$600$}: Fixed-length early commitment. Once traces reach $600$ tokens, it terminates all but the highest-scoring trace and continues generation only for the surviving trace to completion.
        \item \textit{Halve@$300$}: Stage-wise halving. For every $300$ generated tokens, it removes the bottom half of traces according to the current scores, repeating this ``halve-and-continue'' procedure until one trace remains.  
    \end{itemize}
\end{itemize}
We also report Pass@$1$ (single-trace accuracy), Pass@$k$ (oracle success among $k$ traces), and Maj@$k$ (unweighted majority-voting accuracy over $k$ traces). Notably, Pass@$k$ serves as an \emph{upper bound} for any aggregation method that must operate without ground-truth at inference time. While self-consistency and standard Best-of-$N$ require completing all traces with comparable token cost, efficient Best-of-$N$ variants can substantially reduce decoding. For example, with $N=128$ and an average trace length of $6,000$ tokens, the pruning-based variants reduce total generation by nearly $90\%$. On the other hand, OTV introduces negligible computation and token overhead because each verification query is implemented as a single forward pass (see Appendix~\ref{sec:output-length}).

\begin{table*}[t]
\centering
\renewcommand{\arraystretch}{1.0}
\caption{Accuracy (\%) and average output length (in parentheses) on AIME at $N=128$ under Best-of-$N$ and three compute-efficient variants (\ie, Drop@$10$, Stop@$600$, and Halve@$300$). All results are averaged over $64$ runs. }
\vspace{-0.1in}
\begin{small}
\begin{sc}
\resizebox{\textwidth}{!}{
\begin{tabular}{lcccccccc}
\toprule
& \multicolumn{2}{c}{{Best-of-$N$}} & \multicolumn{2}{c}{{Drop@$10$}} &  \multicolumn{2}{c}{{Stop@$600$}} & \multicolumn{2}{c}{{Halve@$300$}}\\
\cmidrule(lr){2-3}
\cmidrule(lr){4-5}
\cmidrule(lr){6-7}
\cmidrule(lr){8-9}
& AIME24 & AIME25 & AIME24 & AIME25 & AIME24 & AIME25  & AIME24 & AIME25 \\
\midrule 
\multicolumn{9}{c}{{Qwen3-4B}} \\
\midrule
DeepConf & 64.95 (9664) & 43.07 (9322) & 62.86 (7044) & 40.78 (7260)& 59.90 (7150) & 40.89 (7084) & 61.98 (6742) & 43.54 (6318)\\
VersaPRM-8B & 54.48 (6560) & 43.28 (6132) & 60.52 (2589) & 43.44 (6334) & \underline{63.75} (3270) & 46.77 (6113) & 59.06 (2981) & 37.24 (6438)\\
Math-Shepherd-7B & \textbf{73.59} (5820) & \textbf{66.51} (5824) & 61.61 (5679) & \underline{46.25} (6116) & 57.45 (5989) & 47.45 (6218) & 54.64 (5722) & 45.36 (6118)\\
Qwen2.5-PRM800K-7B &  69.90 (4891) & 45.10 (6202) & \textbf{66.77} (2196) & 42.86 (6445) &  60.31 (3434) & \textbf{49.53} (6409) & 62.24 (2919) & 45.21 (6619)\\
Qwen2.5-PRM-7B & 71.77 (3720) & 53.33 (3948) & \underline{63.80} (3040) & 45.83 (6588) & \textbf{65.73} (4211) & 44.95 (6304) & \underline{66.46} (3173) & \underline{45.73} (6416)\\
\midrule
OTV (Ours) & \underline{73.44} (5447) & \underline{53.91} (5416) & 63.39 (4427) & \textbf{46.46} (3225)  & \underline{63.75} (4431) & \underline{49.11} (6542) & \textbf{67.03} (4132) & \textbf{49.02} (3170)\\
\midrule
\multicolumn{9}{c}{{DAPO-Qwen-32B}} \\
\midrule
DeepConf & 53.92 (5101) & 38.91 (3957)  & 50.52 (6382) & 37.08 (4398) & 51.82 (7176) & 37.76 (4353) & 50.94 (3772)& 36.77 (4449)\\
VersaPRM-8B & 48.80 (5061) & 31.04 (4447) & \textbf{59.79} (5263) & 39.48 (4796) & 53.12 (5432) & 37.66 (5005) & 49.32 (5744) & 36.61 (5046)\\
Math-Shepherd-7B & \underline{62.34} (5051) & 42.40 (4570) & 55.52 (4819) & 39.22 (4475) & \underline{55.21} (5555) & 36.46 (5063) & \textbf{58.80} (4983) & \underline{42.76} (4919) \\
Qwen2.5-PRM800K-7B & 54.17 (4722) & \textbf{47.81} (4426) & 53.28 (5585) & \underline{41.20} (4689) & 49.48 (5260) & \underline{40.89} (4523) & 47.40 (5732) & 40.00 (4528) \\
Qwen2.5-PRM-7B & 57.03 (4888) & \underline{47.24} (4481) & 55.52 (5525) & 33.65 (5351) & \textbf{55.57} (4939) & 35.31 (5068) & 51.98 (5660) & 36.09 (4967)\\
\midrule
OTV (Ours) &  \textbf{63.18} (4623) & 47.08 (4079) & \underline{55.95} (3397) & \textbf{50.68} (2926) & 53.54 (3211)& \textbf{48.23} (2577) & \underline{55.05} (3436) & \textbf{46.98} (2991)\\
\bottomrule
\end{tabular}
}
\end{sc}
\end{small}
\label{tab:bon}
\end{table*}

\subsection{Main Results}\label{subsec:quanr}
We first evaluate OTV on advanced mathematical reasoning in an \textit{offline} setting, where the complete set of sampled reasoning traces is available and aggregation can be applied post hoc.  
Table~\ref{tab:majority_voting} reports weighted majority-voting results using $128$ traces per run. We weight votes by trace-level confidence scores and discard the lowest-confidence $50\%$ of traces prior to aggregation. 

Across all three backbone scales, OTV consistently delivers the best accuracy on both AIME24 and AIME25. Relative to unweighted majority voting, the gains are substantial (\eg, $+7.9$ points for Qwen3-4B on AIME24 and $+7.8$ points for DAPO-Qwen-32B on AIME25), indicating that OTV's confidence estimates correlate more strongly with trace correctness than heuristic confidence baselines. Compared with other internal methods, OTV is consistently superior, suggesting that a learned verifier that probes the model's internal states provides a more faithful correctness signal than logit-based confidence or self-queried ``yes/no'' verification. Moreover, external verifiers also fall short of OTV across all settings, despite extensive training and broad adoption, highlighting the difficulty of transferring generic reward models across backbones and shifting trace distributions. 

\begin{figure*}[!t]
    \centering  \includegraphics[width=\linewidth]{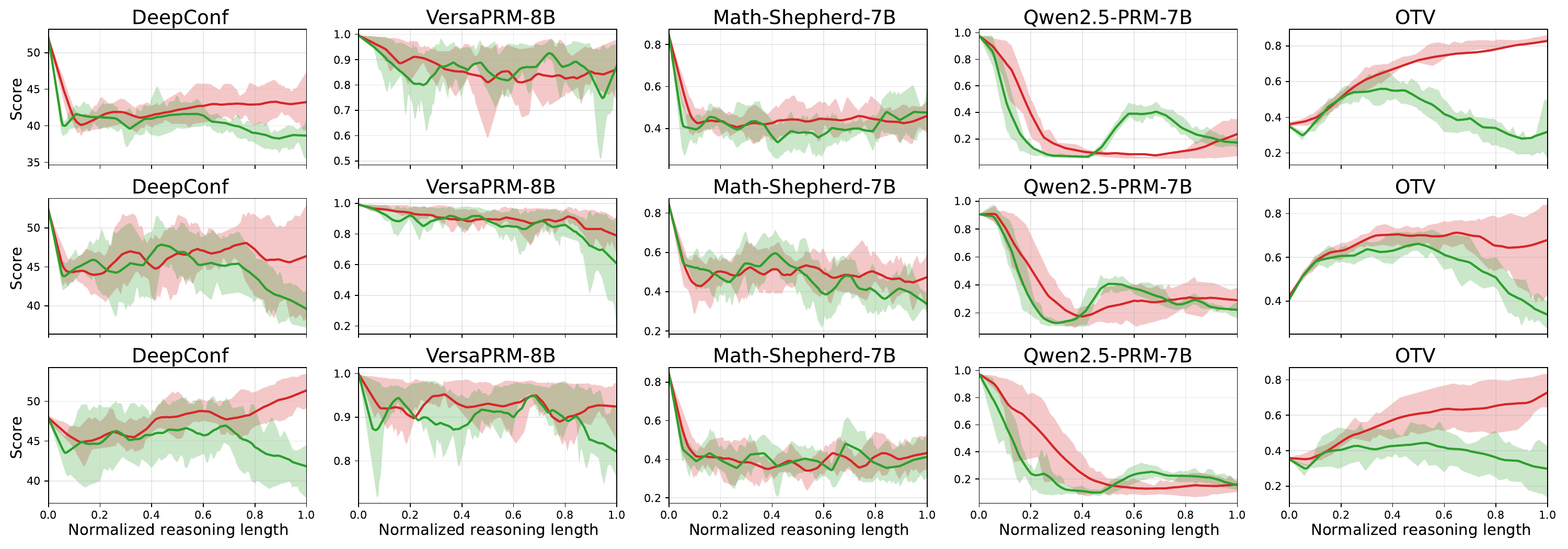}
    \vspace{-0.2in}
    \caption{Confidence dynamics on three representative AIME24 problems (\ie, \#3, \#9, and \#22). For each predictor, we plot the mean confidence trajectory over $32$ sampled reasoning traces, shown separately for traces that end with correct (red) and incorrect (green) final answers. Shaded bands around each mean curve denote the inter-quantile range across traces, summarizing cross-trace variability. }
    \label{fig:trace_dynamics}
\end{figure*}

\begin{figure*}[!t]
    \centering
    \includegraphics[width=1\linewidth]{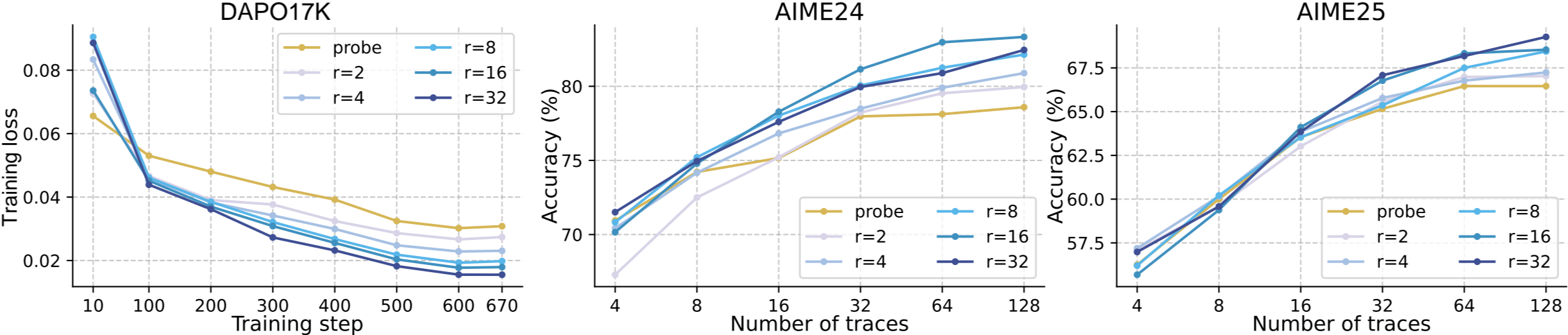}
    \caption{Effect of verifier capacity (\ie, LoRA rank) on training dynamics and downstream voting accuracy. Left: verifier training loss over optimization steps for the ``probe'' baseline, which trains only the regression head (no LoRA; no KV cache) and for OTV with varying LoRA ranks. Middle/Right: Weighted majority-voting accuracy on AIME, as a function of the number of sampled traces. All results are averaged over $64$ runs.}
    \label{fig:lora_rank}
\end{figure*}

Table~\ref{tab:bon} compares OTV with standard Best-of-$N$ and three efficient variants (\ie, Drop@$10$, Stop@$600$, and Halve@$300$) at $N=128$ in an \textit{online} setting. A key finding is that vanilla Best-of-$N$ can be suboptimal in accuracy: averaged across tasks, it lags behind weighted majority voting (at the same decoding cost) by more than $10$ points, reflecting the brittleness of selecting a single trace solely by the highest predicted confidence score. 

Across the efficient variants, OTV attains the best or near-best accuracy in most configurations. Moreover, unlike competing methods, it often yields shorter final traces than standard Best-of-$N$ (approximately $20\%$ fewer tokens on average). This behavior is consistent with OTV's linearly increasing pseudo-confidence target: given two correct traces, the shorter one accumulates confidence faster, reaches a higher score earlier, and is therefore more likely to survive pruning or early-stopping criteria. Among the efficient strategies, Halve@$300$ offers the most favorable accuracy-efficiency trade-off (see also Appendix~\ref{sec:output-length}), where OTV exhibits clear gains. Overall, these results indicate that model-specific verification provides a reliable and compute-efficient alternative to generic score-based selection.

We further validate the generalizability of OTV across pretrained \emph{base} models and diverse architectures (\eg, LLaMA), observing consistent improvements. Details are provided in Appendices~\ref{app:base_model} and~\ref{app:other_models}.

\begin{table*}[!t]
\centering
\renewcommand{\arraystretch}{1.0}
\caption{Ablation of token-level pseudo-confidence labeling heuristics for training OTV.}
\vspace{-0.1in}
\begin{small}
\begin{sc}
\resizebox{\textwidth}{!}{
\begin{tabular}{lcccccc}
\toprule
 & 
 \multicolumn{2}{c}{Maj@$128$} &  \multicolumn{2}{c}{{Best-of-$N$}} & \multicolumn{2}{c}{{Halve@$300$}}\\
\cmidrule(lr){2-3}
\cmidrule(lr){4-5}
\cmidrule(lr){6-7}
& AIME24 & AIME25 & AIME24 & AIME25  & AIME24 & AIME25 \\
\midrule 
Constant label & 81.41\small{$\pm$1.65} & 68.49\small{$\pm$2.04} &  72.76 ({4,340})& 54.27 ({4,362}) & 66.41 ({3,951}) & \textbf{49.27} (4,503)\\
Sigmoid ramp & 80.05\small{$\pm$1.91} & 66.51\small{$\pm$0.92} & \textbf{77.40} (4,617)&  \textbf{56.41} (4,426) & 68.44 (4,210) & 46.30 (4,294)\\
Noise-perturbed ramp & 81.87\small{$\pm$1.65} &  67.92\small{$\pm$1.61} & 72.60 (5,226) & 54.90 (5,152) & \textbf{71.30} (4,134) & 48.33 (5,736)\\
Stepwise ramp & \underline{82.40}\small{$\pm$1.50} & \underline{68.80}\small{$\pm$1.60} & \underline{77.29} (5,351) & \underline{55.36} (5,452) & \underline{70.63} (4,231) & 37.60 (7,095)\\
Linear ramp (Default)   & \textbf{83.33}\small{$\pm$1.57} & \textbf{69.32}\small{$\pm$1.46}& 73.44 (5,447) & 53.91 (5,416) & 67.03 (4,132) & \underline{49.02} ({3,170}) \\
\bottomrule
\end{tabular}
}
\vspace{-0.05in}
\end{sc}
\end{small}
\label{tab:priors}
\end{table*}

\begin{table*}[!t]
\centering
\renewcommand{\arraystretch}{1.0}
\caption{Effect of trace-level confidence aggregation and filtering on weighted majority voting. We convert OTV's token-level confidence estimates into a single trace-level score by aggregating over the last few tokens, discard the lowest-confidence $\rho$ fraction of traces (\ie, $0\%$, $25\%$, $50\%$, and $75\%$), and evaluate three aggregation operators.}
\vspace{-0.1in}
\begin{small}
\begin{sc}
\resizebox{\textwidth}{!}{
\begin{tabular}{lcccccccccccc}
\toprule
 & \multicolumn{4}{c}{$\max$} & \multicolumn{4}{c}{$\mathrm{mean}$} &  \multicolumn{4}{c}{$\min$}\\
\cmidrule(lr){2-5}
\cmidrule(lr){6-9}
\cmidrule(lr){10-13}
& $0\%$ & $25\%$ & $50\%$ & $75\%$ & $0\%$ & $25\%$ & $50\%$ & $75\%$ & $0\%$ & $25\%$ & $50\%$ & $75\%$\\
\midrule 
All tokens &75.78 & 76.82& 80.42& 80.36& 75.94 & 79.53&82.19& 80.00 & 80.10 & 80.83 & 81.25 & 79.17 \\
Last $10\%$ tokens & 76.77 & 80.21 &82.76 & 81.46 & 82.19 & 83.02 & 83.28 & 82.19 & 83.07 & 83.07 & 83.12 & 82.29\\
Last $1,600$ tokens & 75.78& 78.65& 81.25& 82.14& 80.99 & 82.71 & 83.23 & 82.24   & 83.23 & 83.23 & 82.97 & 79.95 \\
Last $400$ tokens& 77.50 &80.36 &82.66 &82.86 & 81.93 & 82.81 & 83.33 & 83.23 & 82.86 & 83.23 & 83.28 & \underline{83.75}\\
Last $100$ tokens & 78.39 & 80.21 & 82.34 & 81.82 & 82.08 & 82.97 & 83.33 & 82.76 & 82.86 & 82.60 & 82.76 & \textbf{83.96}\\
\bottomrule
\end{tabular}
}
\end{sc}
\end{small}
\label{tab:filter}
\end{table*}

\subsection{Visualization}
Figure~\ref{fig:trace_dynamics} provides a qualitative view of how different confidence predictors evolve throughout generation on three representative AIME 24 problems (\ie, \#3, \#9, and \#22). We plot, for each verifier, the \textit{mean} confidence trajectory over $32$ sampled traces, shown separately for correct (red) and incorrect (green) solutions. The shaded region around each mean curve reflects cross-trace variability (\ie, the $0.2$–$0.8$ inter-quantile range), facilitating comparisons of overall trends across verifiers. Additional OTV visualizations for all AIME problems are provided in Appendix~\ref{appendix:morevis}.

A consistent pattern emerges across all three cases: DeepConf and prior process reward models often produce highly entangled confidence curves, with substantial overlap between correct and incorrect traces, which limits their ability to discriminate promising candidates from failures early in decoding. In contrast, OTV exhibits noticeably clearer stratification: confidence typically increases over the course of reasoning for correct traces, whereas incorrect traces remain comparatively suppressed. This enhanced separability aligns with the quantitative gains reported in Sec.~\ref{subsec:quanr}, and helps explain OTV's stronger performance under score-based selection and pruning and early-termination regimes. 

Appendix~G provides additional token-level visualizations that align confidence curves with the generated text. This finer-grained analysis illustrates how OTV's scores evolve within different solutions---often rising sharply after key computational steps---thereby improving interpretability at the granularity of individual reasoning steps.

\subsection{Ablation Studies}
\label{sec:ablation}
We conduct ablation studies to quantify how individual design choices contribute to OTV's performance. Specifically, we examine 1) the LoRA rank used in the verification module, 2) alternative pseudo-confidence labeling rules, and 3) the impact of trace-level confidence aggregation and filtering on weighted majority voting.

\paragraph{Effect of LoRA rank}
We first vary the LoRA rank $r$ while keeping the underlying reasoner fixed, and evaluate performance under weighted majority voting as the number of sampled traces increases. As shown in Figure~\ref{fig:lora_rank}, we observe a clear capacity-performance trade-off: increasing $r$ reduces the verifier's training loss and improves downstream accuracy, with substantial gains already at moderate ranks (\eg, $r=16$). Importantly, OTV consistently outperforms the ``probe'' baseline, which trains only the regression head (\ie, without LoRA fine-tuning or KV cache reuse). This suggests that the verifier benefits from both 1) additional adaptation capacity and 2) richer access to the reasoning trajectory via the KV cache, rather than relying solely on final-layer hidden states~\citep{zhang2025reasoning}.

\paragraph{Pseudo-confidence labeling}
We next ablate the heuristic pseudo-confidence targets $c_t$ defined in Sec.~\ref{sec_3}, while keeping the voting protocol fixed. Table~\ref{tab:priors} shows that \textit{constant label}, which applies a hindsight-style target uniformly over the entire trace, tends to bias the verifier toward shorter traces and slightly degrades the accuracy of weighted majority voting. In contrast, ramp-based alternatives achieve comparable performance while mitigating this uniform-hindsight bias by imposing a monotonic progression toward the final outcome. We thus adopt the \textit{linear ramp} as the default due to its simplicity and stable performance. 

\paragraph{Trace-level confidence aggregation and filtering} We further examine how to map token-level confidence estimates to a single trace-level score for weighted majority voting. Following the DeepConf-style protocol~\citep{fu2025deep}, we aggregate confidence over the last segment of each trace and then remove the bottom $\rho$ fraction of traces before voting. Table~\ref{tab:filter} reveals two consistent trends. First, aggregating over the tail of the trajectory (\eg, the last $100$ tokens) is more reliable than using the full trace, supporting the intuition that late-step verification better captures the fully formed reasoning state and final answer. Second, moderate-to-aggressive filtering improves robustness by suppressing low-confidence candidates, with mean/min aggregation over roughly the last $100$ tokens performing the best in our setting.

\section{Conclusion}
In this work, we 
have introduced OTV, an efficient, model-specific computational method for estimating token-level reasoning correctness. OTV equips a reasoning LLM with a LoRA-gated verification pathway that is activated only when a dedicated \texttt{[ToT]} is inserted, allowing the model to enter verification mode without perturbing its default reasoning behavior. By probing the model's KV cache and producing a scalar confidence score through a small regression head, OTV supports anytime verification at the cost of just a single forward pass per query. Experiments across multiple reasoning LLMs and math benchmarks show that OTV consistently improves parallel-thinking aggregation and pruning strategies relative to existing internal and external verifiers, while delivering substantial efficiency gains via confidence-guided early termination. We further discuss limitations and future research directions in Appendix~\ref{app:future}.

\section*{Impact Statement}
This paper presents work with the goal of advancing the field of machine learning. There are many potential societal consequences of our work, none of which we feel need to be specifically highlighted here.


\begin{thebibliography}{75}
\providecommand{\natexlab}[1]{#1}
\providecommand{\url}[1]{\texttt{#1}}
\expandafter\ifx\csname urlstyle\endcsname\relax
  \providecommand{\doi}[1]{doi: #1}\else
  \providecommand{\doi}{doi: \begingroup \urlstyle{rm}\Url}\fi

\bibitem[Ankner et~al.(2024)Ankner, Paul, Cui, Chang, and Ammanabrolu]{ankner2024critique}
Ankner, Z., Paul, M., Cui, B., Chang, J.~D., and Ammanabrolu, P.
\newblock {Critique-out-Loud} reward models.
\newblock \emph{arXiv preprint arXiv:2408.11791}, 2024.

\bibitem[Azaria \& Mitchell(2023)Azaria and Mitchell]{azaria2023internal}
Azaria, A. and Mitchell, T.
\newblock The internal state of an {LLM} knows when it's lying.
\newblock \emph{arXiv preprint arXiv:2304.13734}, 2023.

\bibitem[Brown et~al.(2024)Brown, Juravsky, Ehrlich, Clark, Le, R{\'e}, and Mirhoseini]{brown2024large}
Brown, B., Juravsky, J., Ehrlich, R., Clark, R., Le, Q.~V., R{\'e}, C., and Mirhoseini, A.
\newblock Large language monkeys: Scaling inference compute with repeated sampling.
\newblock \emph{arXiv preprint arXiv:2407.21787}, 2024.

\bibitem[Burns et~al.(2022)Burns, Ye, Klein, and Steinhardt]{burns2022discovering}
Burns, C., Ye, H., Klein, D., and Steinhardt, J.
\newblock Discovering latent knowledge in language models without supervision.
\newblock \emph{arXiv preprint arXiv:2212.03827}, 2022.

\bibitem[Chen et~al.(2021)Chen, Tworek, Jun, Yuan, Pinto, Kaplan, Edwards, Burda, Joseph, Brockman, et~al.]{chen2021evaluating}
Chen, M., Tworek, J., Jun, H., Yuan, Q., Pinto, H. P.~O., Kaplan, J., Edwards, H., Burda, Y., Joseph, N., Brockman, G., et~al.
\newblock Evaluating large language models trained on code.
\newblock \emph{arXiv preprint arXiv:2107.03374}, 2021.

\bibitem[Chen et~al.(2024{\natexlab{a}})Chen, Xu, Liang, He, Pang, Yu, Song, Liu, Zhou, Zhang, et~al.]{chen2024not}
Chen, X., Xu, J., Liang, T., He, Z., Pang, J., Yu, D., Song, L., Liu, Q., Zhou, M., Zhang, Z., et~al.
\newblock Do {NOT} think that much for 2+3=? {O}n the overthinking of {o1}-like {LLMs}.
\newblock \emph{arXiv preprint arXiv:2412.21187}, 2024{\natexlab{a}}.

\bibitem[Chen et~al.(2024{\natexlab{b}})Chen, White, Mooney, Payani, Su, and Sun]{chen2024tree}
Chen, Z., White, M., Mooney, R., Payani, A., Su, Y., and Sun, H.
\newblock When is tree search useful for {LLM} planning? {It} depends on the discriminator.
\newblock \emph{arXiv preprint arXiv:2402.10890}, 2024{\natexlab{b}}.

\bibitem[Cobbe et~al.(2021)Cobbe, Kosaraju, Bavarian, Chen, Jun, Kaiser, Plappert, Tworek, Hilton, Nakano, et~al.]{cobbe2021training}
Cobbe, K., Kosaraju, V., Bavarian, M., Chen, M., Jun, H., Kaiser, L., Plappert, M., Tworek, J., Hilton, J., Nakano, R., et~al.
\newblock Training verifiers to solve math word problems.
\newblock \emph{arXiv preprint arXiv:2110.14168}, 2021.

\bibitem[Comanici et~al.(2025)Comanici, Bieber, Schaekermann, Pasupat, Sachdeva, Dhillon, Blistein, Ram, Zhang, Rosen, et~al.]{comanici2025gemini}
Comanici, G., Bieber, E., Schaekermann, M., Pasupat, I., Sachdeva, N., Dhillon, I., Blistein, M., Ram, O., Zhang, D., Rosen, E., et~al.
\newblock {Gemini 2.5}: Pushing the frontier with advanced reasoning, multimodality, long context, and next generation agentic capabilities.
\newblock \emph{arXiv preprint arXiv:2507.06261}, 2025.

\bibitem[Fadeeva et~al.(2024)Fadeeva, Rubashevskii, Shelmanov, Petrakov, Li, Mubarak, Tsymbalov, Kuzmin, Panchenko, Baldwin, et~al.]{fadeeva2024fact}
Fadeeva, E., Rubashevskii, A., Shelmanov, A., Petrakov, S., Li, H., Mubarak, H., Tsymbalov, E., Kuzmin, G., Panchenko, A., Baldwin, T., et~al.
\newblock Fact-checking the output of large language models via token-level uncertainty quantification.
\newblock \emph{arXiv preprint arXiv:2403.04696}, 2024.

\bibitem[Feng et~al.(2024)Feng, Kong, Ma, Zhang, Yin, Wang, Pang, and Yang]{feng2024step}
Feng, S., Kong, X., Ma, S., Zhang, A., Yin, D., Wang, C., Pang, R., and Yang, Y.
\newblock Step-by-step reasoning for math problems via twisted sequential {Monte Carlo}.
\newblock \emph{arXiv preprint arXiv:2410.01920}, 2024.

\bibitem[Fu et~al.(2025)Fu, Wang, Tian, and Zhao]{fu2025deep}
Fu, Y., Wang, X., Tian, Y., and Zhao, J.
\newblock Deep think with confidence.
\newblock \emph{arXiv preprint arXiv:2508.15260}, 2025.

\bibitem[Ghosal et~al.(2025)Ghosal, Chakraborty, Reddy, Lu, Wang, Manocha, Huang, Ghavamzadeh, and Bedi]{ghosal2025does}
Ghosal, S.~S., Chakraborty, S., Reddy, A., Lu, Y., Wang, M., Manocha, D., Huang, F., Ghavamzadeh, M., and Bedi, A.~S.
\newblock Does thinking more always help? {Understanding} test-time scaling in reasoning models.
\newblock \emph{arXiv preprint arXiv:2506.04210}, 2025.

\bibitem[Golovneva et~al.(2022)Golovneva, Chen, Poff, Corredor, Zettlemoyer, Fazel-Zarandi, and Celikyilmaz]{golovneva2022roscoe}
Golovneva, O., Chen, M., Poff, S., Corredor, M., Zettlemoyer, L., Fazel-Zarandi, M., and Celikyilmaz, A.
\newblock {ROSCOE}: A suite of metrics for scoring step-by-step reasoning.
\newblock \emph{arXiv preprint arXiv:2212.07919}, 2022.

\bibitem[Guan et~al.(2025)Guan, Zhang, Liu, Shang, Sun, Zhu, Yang, and Yang]{guan2025rstar}
Guan, X., Zhang, L.~L., Liu, Y., Shang, N., Sun, Y., Zhu, Y., Yang, F., and Yang, M.
\newblock {rStar-Math}: Small {LLMs} can master math reasoning with self-evolved deep thinking.
\newblock \emph{arXiv preprint arXiv:2501.04519}, 2025.

\bibitem[Guo et~al.(2025)Guo, Yang, Zhang, Song, Zhang, Xu, Zhu, Ma, Wang, Bi, et~al.]{guo2025deepseek}
Guo, D., Yang, D., Zhang, H., Song, J., Zhang, R., Xu, R., Zhu, Q., Ma, S., Wang, P., Bi, X., et~al.
\newblock Deepseek-{R1}: Incentivizing reasoning capability in {LLMs} via reinforcement learning.
\newblock \emph{arXiv preprint arXiv:2501.12948}, 2025.

\bibitem[Hosseini et~al.(2024)Hosseini, Yuan, Malkin, Courville, Sordoni, and Agarwal]{hosseini2024v}
Hosseini, A., Yuan, X., Malkin, N., Courville, A., Sordoni, A., and Agarwal, R.
\newblock {V-STaR}: Training verifiers for self-taught reasoners.
\newblock In \emph{Conference on Language Modeling}, 2024.

\bibitem[Hsu et~al.(2025)Hsu, Buffelli, McGowan, Liao, Chen, Vakili, and Shiu]{hsu2025group}
Hsu, C.-J., Buffelli, D., McGowan, J., Liao, F.-T., Chen, Y.-C., Vakili, S., and Shiu, D.-s.
\newblock {Group Think}: Multiple concurrent reasoning agents collaborating at token level granularity.
\newblock \emph{arXiv preprint arXiv:2505.11107}, 2025.

\bibitem[Hu et~al.(2022)Hu, Shen, Wallis, Allen-Zhu, Li, Wang, Wang, and Chen]{hu2021lora}
Hu, E.~J., Shen, Y., Wallis, P., Allen-Zhu, Z., Li, Y., Wang, S., Wang, L., and Chen, W.
\newblock {LoRA}: Low-rank adaptation of large language models.
\newblock In \emph{International Conference on Learning Representations}, 2022.

\bibitem[Huang et~al.(2025)Huang, Huang, Leng, Liu, and Huang]{huang2025efficient}
Huang, C., Huang, L., Leng, J., Liu, J., and Huang, J.
\newblock Efficient test-time scaling via self-calibration.
\newblock \emph{arXiv preprint arXiv:2503.00031}, 2025.

\bibitem[Huang et~al.(2023)Huang, Chen, Mishra, Zheng, Yu, Song, and Zhou]{huang2023large}
Huang, J., Chen, X., Mishra, S., Zheng, H.~S., Yu, A.~W., Song, X., and Zhou, D.
\newblock Large language models cannot self-correct reasoning yet.
\newblock \emph{arXiv preprint arXiv:2310.01798}, 2023.

\bibitem[Jaech et~al.(2024)Jaech, Kalai, Lerer, Richardson, El-Kishky, Low, Helyar, Madry, Beutel, Carney, et~al.]{jaech2024openai}
Jaech, A., Kalai, A., Lerer, A., Richardson, A., El-Kishky, A., Low, A., Helyar, A., Madry, A., Beutel, A., Carney, A., et~al.
\newblock {OpenAI} o1 system card.
\newblock \emph{arXiv preprint arXiv:2412.16720}, 2024.

\bibitem[Kang et~al.(2025)Kang, Zhao, and Song]{kang2025scalable}
Kang, Z., Zhao, X., and Song, D.
\newblock Scalable best-of-{$N$} selection for large language models via self-certainty.
\newblock \emph{arXiv preprint arXiv:2502.18581}, 2025.

\bibitem[Lee \& Hockenmaier(2025)Lee and Hockenmaier]{lee2025evaluating}
Lee, J. and Hockenmaier, J.
\newblock Evaluating step-by-step reasoning traces: A survey.
\newblock \emph{arXiv preprint arXiv:2502.12289}, 2025.

\bibitem[Lee et~al.(2025)Lee, Yang, Heo, Han, Kim, Yang, and Yoo]{lee2025tokensupervised}
Lee, J.~H., Yang, J.~Y., Heo, B., Han, D., Kim, K., Yang, E., and Yoo, K.~M.
\newblock Token-supervised value models for enhancing mathematical problem-solving capabilities of large language models.
\newblock In \emph{International Conference on Learning Representations}, 2025.

\bibitem[Lewis et~al.(2020)Lewis, Perez, Piktus, Petroni, Karpukhin, Goyal, K{\"u}ttler, Lewis, Yih, Rockt{\"a}schel, et~al.]{lewis2020retrieval}
Lewis, P., Perez, E., Piktus, A., Petroni, F., Karpukhin, V., Goyal, N., K{\"u}ttler, H., Lewis, M., Yih, W.-t., Rockt{\"a}schel, T., et~al.
\newblock Retrieval-augmented generation for knowledge-intensive {NAIME2024LP} tasks.
\newblock In \emph{Advances in Neural Information Processing Systems}, pp.\  9459--9474, 2020.

\bibitem[Li et~al.(2025)Li, Zhou, Muhtar, Yin, Yan, Shen, Liang, Vosoughi, and Liu]{li2025diffusion}
Li, P., Zhou, Y., Muhtar, D., Yin, L., Yan, S., Shen, L., Liang, Y., Vosoughi, S., and Liu, S.
\newblock Diffusion language models know the answer before decoding.
\newblock \emph{arXiv preprint arXiv:2508.19982}, 2025.

\bibitem[Lifshitz et~al.(2025)Lifshitz, McIlraith, and Du]{lifshitz2025multi}
Lifshitz, S., McIlraith, S.~A., and Du, Y.
\newblock Multi-agent verification: Scaling test-time compute with multiple verifiers.
\newblock \emph{arXiv preprint arXiv:2502.20379}, 2025.

\bibitem[Lightman et~al.(2024)Lightman, Kosaraju, Burda, Edwards, Baker, Lee, Leike, Schulman, Sutskever, and Cobbe]{lightman2023let}
Lightman, H., Kosaraju, V., Burda, Y., Edwards, H., Baker, B., Lee, T., Leike, J., Schulman, J., Sutskever, I., and Cobbe, K.
\newblock Let's verify step by step.
\newblock In \emph{International Conference on Learning Representations}, 2024.

\bibitem[Lin et~al.(2022)Lin, Hilton, and Evans]{lin2022teaching}
Lin, S., Hilton, J., and Evans, O.
\newblock Teaching models to express their uncertainty in words.
\newblock \emph{arXiv preprint arXiv:2205.14334}, 2022.

\bibitem[Liu et~al.(2024)Liu, Chen, Shoeybi, Catanzaro, and Ping]{liu2024acemath}
Liu, Z., Chen, Y., Shoeybi, M., Catanzaro, B., and Ping, W.
\newblock {AceMath}: Advancing frontier math reasoning with post-training and reward modeling.
\newblock \emph{arXiv preprint arXiv:2412.15084}, 2024.

\bibitem[Lu et~al.(2024)Lu, Dou, Wang, Cao, Dai, Feng, and Guo]{lu2024autopsv}
Lu, J., Dou, Z., Wang, H., Cao, Z., Dai, J., Feng, Y., and Guo, Z.
\newblock {AutoPSV}: Automated process-supervised verifier.
\newblock In \emph{Advances in Neural Information Processing Systems}, pp.\  79935--79962, 2024.

\bibitem[Luo et~al.(2024)Luo, Liu, Liu, Phatale, Guo, Lara, Li, Shu, Zhu, Meng, et~al.]{luo2024improve}
Luo, L., Liu, Y., Liu, R., Phatale, S., Guo, M., Lara, H., Li, Y., Shu, L., Zhu, Y., Meng, L., et~al.
\newblock Improve mathematical reasoning in language models by automated process supervision.
\newblock \emph{arXiv preprint arXiv:2406.06592}, 2024.

\bibitem[{MAA}(2024)]{AIME2024}
{MAA}.
\newblock 2024 {American Invitational Mathematics Examination (AIME)}.
\newblock Competition Problems and Solutions, 2024.
\newblock URL \url{https://www.maa.org/math-competitions/aime}.

\bibitem[{MAA}(2025)]{AIME2025}
{MAA}.
\newblock 2025 {American Invitational Mathematics Examination (AIME)}.
\newblock Competition Problems and Solutions, 2025.
\newblock URL \url{https://www.maa.org/math-competitions/aime}.

\bibitem[Muennighoff et~al.(2025)Muennighoff, Yang, Shi, Li, Fei-Fei, Hajishirzi, Zettlemoyer, Liang, Cand{\`e}s, and Hashimoto]{muennighoff2025s1}
Muennighoff, N., Yang, Z., Shi, W., Li, X.~L., Fei-Fei, L., Hajishirzi, H., Zettlemoyer, L., Liang, P., Cand{\`e}s, E., and Hashimoto, T.
\newblock {s1}: Simple test-time scaling.
\newblock \emph{arXiv preprint arXiv:2501.19393}, 2025.

\bibitem[Ouyang et~al.(2022)Ouyang, Wu, Jiang, Almeida, Wainwright, Mishkin, Zhang, Agarwal, Slama, Ray, et~al.]{ouyang2022training}
Ouyang, L., Wu, J., Jiang, X., Almeida, D., Wainwright, C., Mishkin, P., Zhang, C., Agarwal, S., Slama, K., Ray, A., et~al.
\newblock Training language models to follow instructions with human feedback.
\newblock In \emph{Advances in Neural Information Processing Systems}, pp.\  27730--27744, 2022.

\bibitem[Samragh et~al.(2025)Samragh, Kundu, Harrison, Nishu, Naik, Cho, and Farajtabar]{samragh2025your}
Samragh, M., Kundu, A., Harrison, D., Nishu, K., Naik, D., Cho, M., and Farajtabar, M.
\newblock Your {LLM} knows the future: Uncovering its multi-token prediction potential.
\newblock \emph{arXiv preprint arXiv:2507.11851}, 2025.

\bibitem[Schick et~al.(2023)Schick, Dwivedi-Yu, Dess{\`\i}, Raileanu, Lomeli, Hambro, Zettlemoyer, Cancedda, and Scialom]{schick2023toolformer}
Schick, T., Dwivedi-Yu, J., Dess{\`\i}, R., Raileanu, R., Lomeli, M., Hambro, E., Zettlemoyer, L., Cancedda, N., and Scialom, T.
\newblock Toolformer: Language models can teach themselves to use tools.
\newblock In \emph{Advances in Neural Information Processing Systems}, pp.\  68539--68551, 2023.

\bibitem[Setlur et~al.(2024)Setlur, Nagpal, Fisch, Geng, Eisenstein, Agarwal, Agarwal, Berant, and Kumar]{setlur2024rewarding}
Setlur, A., Nagpal, C., Fisch, A., Geng, X., Eisenstein, J., Agarwal, R., Agarwal, A., Berant, J., and Kumar, A.
\newblock Rewarding progress: Scaling automated process verifiers for {LLM} reasoning.
\newblock \emph{arXiv preprint arXiv:2410.08146}, 2024.

\bibitem[Shao et~al.(2024)Shao, Wang, Zhu, Xu, Song, Bi, Zhang, Zhang, Li, Wu, et~al.]{shao2024deepseekmath}
Shao, Z., Wang, P., Zhu, Q., Xu, R., Song, J., Bi, X., Zhang, H., Zhang, M., Li, Y., Wu, Y., et~al.
\newblock {DeepSeekMath}: Pushing the limits of mathematical reasoning in open language models.
\newblock \emph{arXiv preprint arXiv:2402.03300}, 2024.

\bibitem[Stiennon et~al.(2020)Stiennon, Ouyang, Wu, Ziegler, Lowe, Voss, Radford, Amodei, and Christiano]{stiennon2020learning}
Stiennon, N., Ouyang, L., Wu, J., Ziegler, D., Lowe, R., Voss, C., Radford, A., Amodei, D., and Christiano, P.~F.
\newblock Learning to summarize with human feedback.
\newblock In \emph{Advances in Neural Information Processing Systems}, pp.\  3008--3021, 2020.

\bibitem[Suzgun et~al.(2023)Suzgun, Scales, Sch{\"a}rli, Gehrmann, Tay, Chung, Chowdhery, Le, Chi, Zhou, et~al.]{suzgun2023challenging}
Suzgun, M., Scales, N., Sch{\"a}rli, N., Gehrmann, S., Tay, Y., Chung, H.~W., Chowdhery, A., Le, Q.~V., Chi, E.~H., Zhou, D., et~al.
\newblock Challenging {BIG-Bench} tasks and whether chain-of-thought can solve them.
\newblock In \emph{Findings of the Association for Computational Linguistics}, pp.\  13003--13051, 2023.

\bibitem[Team et~al.(2025)Team, Du, Gao, Xing, Jiang, Chen, Li, Xiao, Du, Liao, et~al.]{team2025kimi}
Team, K., Du, A., Gao, B., Xing, B., Jiang, C., Chen, C., Li, C., Xiao, C., Du, C., Liao, C., et~al.
\newblock Kimi k1.5: Scaling reinforcement learning with {LLMs}.
\newblock \emph{arXiv preprint arXiv:2501.12599}, 2025.

\bibitem[Uesato et~al.(2022)Uesato, Kushman, Kumar, Song, Siegel, Wang, Creswell, Irving, and Higgins]{uesato2022solving}
Uesato, J., Kushman, N., Kumar, R., Song, F., Siegel, N., Wang, L., Creswell, A., Irving, G., and Higgins, I.
\newblock Solving math word problems with process-and outcome-based feedback.
\newblock \emph{arXiv preprint arXiv:2211.14275}, 2022.

\bibitem[Vaswani et~al.(2017)Vaswani, Shazeer, Parmar, Uszkoreit, Jones, Gomez, Kaiser, and Polosukhin]{vaswani2017attention}
Vaswani, A., Shazeer, N., Parmar, N., Uszkoreit, J., Jones, L., Gomez, A.~N., Kaiser, {\L}., and Polosukhin, I.
\newblock Attention is all you need.
\newblock In \emph{Advances in Neural Information Processing Systems}, 2017.

\bibitem[Venktesh et~al.(2025)Venktesh, Rathee, and Anand]{venktesh2025trust}
Venktesh, V., Rathee, M., and Anand, A.
\newblock Trust but verify! {A} survey on verification design for test-time scaling.
\newblock \emph{arXiv preprint arXiv:2508.16665}, 2025.

\bibitem[Wang et~al.(2024)Wang, Xiong, Xie, Zhao, and Zhang]{wang2024interpretable}
Wang, H., Xiong, W., Xie, T., Zhao, H., and Zhang, T.
\newblock Interpretable preferences via multi-objective reward modeling and mixture-of-experts.
\newblock \emph{arXiv preprint arXiv:2406.12845}, 2024.

\bibitem[Wang et~al.(2023{\natexlab{a}})Wang, Li, Shao, Xu, Dai, Li, Chen, Wu, and Sui]{wang2023math}
Wang, P., Li, L., Shao, Z., Xu, R., Dai, D., Li, Y., Chen, D., Wu, Y., and Sui, Z.
\newblock {Math-Shepherd}: Verify and reinforce {LLMs} step-by-step without human annotations.
\newblock \emph{arXiv preprint arXiv:2312.08935}, 2023{\natexlab{a}}.

\bibitem[Wang et~al.(2023{\natexlab{b}})Wang, Wei, Schuurmans, Le, Chi, Narang, Chowdhery, and Zhou]{wang2022self}
Wang, X., Wei, J., Schuurmans, D., Le, Q.~V., Chi, E.~H., Narang, S., Chowdhery, A., and Zhou, D.
\newblock Self-consistency improves chain of thought reasoning in language models.
\newblock In \emph{International Conference on Learning Representations}, 2023{\natexlab{b}}.

\bibitem[Wang et~al.(2025)Wang, Zhang, Huang, Yang, Zhang, Huang, and Wang]{wang2025sampling}
Wang, Y., Zhang, P., Huang, S., Yang, B., Zhang, Z., Huang, F., and Wang, R.
\newblock Sampling-efficient test-time scaling: Self-estimating the best-of-{$N$} sampling in early decoding.
\newblock \emph{arXiv preprint arXiv:2503.01422}, 2025.

\bibitem[Wei et~al.(2022)Wei, Wang, Schuurmans, Bosma, Xia, Chi, Le, and Zhou]{wei2022chain}
Wei, J., Wang, X., Schuurmans, D., Bosma, M., Xia, F., Chi, E.~H., Le, Q.~V., and Zhou, D.
\newblock Chain-of-thought prompting elicits reasoning in large language models.
\newblock In \emph{Advances in Neural Information Processing Systems}, pp.\  24824--24837, 2022.

\bibitem[Wen et~al.(2025)Wen, Su, Zhang, Liu, Liu, Zhang, and Li]{wen2025parathinker}
Wen, H., Su, Y., Zhang, F., Liu, Y., Liu, Y., Zhang, Y.-Q., and Li, Y.
\newblock {ParaThinker}: Native parallel thinking as a new paradigm to scale {LLM} test-time compute.
\newblock \emph{arXiv preprint arXiv:2509.04475}, 2025.

\bibitem[Xiong et~al.(2024)Xiong, Hu, Lu, Li, Fu, He, and Hooi]{xiongcan}
Xiong, M., Hu, Z., Lu, X., Li, Y., Fu, J., He, J., and Hooi, B.
\newblock Can {LLMs} express their uncertainty? {A}n empirical evaluation of confidence elicitation in {LLMs}.
\newblock In \emph{International Conference on Learning Representations}, 2024.

\bibitem[Yang et~al.(2024)Yang, Zhang, Hui, Gao, Yu, Li, Liu, Tu, Zhou, Lin, et~al.]{yang2024qwen25mathtechnicalreportmathematical}
Yang, A., Zhang, B., Hui, B., Gao, B., Yu, B., Li, C., Liu, D., Tu, J., Zhou, J., Lin, J., et~al.
\newblock {Qwen2.5-Math} technical report: Toward mathematical expert model via self-improvement.
\newblock \emph{arXiv preprint arXiv:2409.12122}, 2024.

\bibitem[Yang et~al.(2025{\natexlab{a}})Yang, Li, Yang, Zhang, Hui, Zheng, Yu, Gao, Huang, Lv, et~al.]{yang2025qwen3}
Yang, A., Li, A., Yang, B., Zhang, B., Hui, B., Zheng, B., Yu, B., Gao, C., Huang, C., Lv, C., et~al.
\newblock Qwen3 technical report.
\newblock \emph{arXiv preprint arXiv:2505.09388}, 2025{\natexlab{a}}.

\bibitem[Yang et~al.(2025{\natexlab{b}})Yang, An, Liu, Chen, and Chen]{yang2025multiverse}
Yang, X., An, Y., Liu, H., Chen, T., and Chen, B.
\newblock Multiverse: Your language models secretly decide how to parallelize and merge generation.
\newblock \emph{arXiv preprint arXiv:2506.09991}, 2025{\natexlab{b}}.

\bibitem[Yao et~al.(2023)Yao, Yu, Zhao, Shafran, Griffiths, Cao, and Narasimhan]{yao2023tree}
Yao, S., Yu, D., Zhao, J., Shafran, I., Griffiths, T., Cao, Y., and Narasimhan, K.
\newblock Tree of thoughts: Deliberate problem solving with large language models.
\newblock In \emph{Advances in Neural Information Processing Systems}, pp.\  11809--11822, 2023.

\bibitem[Ye et~al.(2025)Ye, Melo, Kaddar, Blunsom, Staton, and Gal]{ye2025uncertainty}
Ye, Z., Melo, L.~C., Kaddar, Y., Blunsom, P., Staton, S., and Gal, Y.
\newblock Uncertainty-aware step-wise verification with generative reward models.
\newblock \emph{arXiv preprint arXiv:2502.11250}, 2025.

\bibitem[Yu et~al.(2023)Yu, Gao, and Wang]{yu2023ovm}
Yu, F., Gao, A., and Wang, B.
\newblock {OVM}, outcome-supervised value models for planning in mathematical reasoning.
\newblock \emph{arXiv preprint arXiv:2311.09724}, 2023.

\bibitem[Yu et~al.(2024)Yu, Jiang, Shi, Yu, Liu, Zhang, Kwok, Li, Weller, and Liu]{yu2023metamath}
Yu, L., Jiang, W., Shi, H., Yu, J., Liu, Z., Zhang, Y., Kwok, J.~T., Li, Z., Weller, A., and Liu, W.
\newblock {MetaMath}: Bootstrap your own mathematical questions for large language models.
\newblock In \emph{International Conference on Learning Representations}, 2024.

\bibitem[Yu et~al.(2025)Yu, Zhang, Zhu, Yuan, Zuo, Yue, Dai, Fan, Liu, Liu, et~al.]{yu2025dapo}
Yu, Q., Zhang, Z., Zhu, R., Yuan, Y., Zuo, X., Yue, Y., Dai, W., Fan, T., Liu, G., Liu, L., et~al.
\newblock {DAPO}: An open-source {LLM} reinforcement learning system at scale.
\newblock \emph{arXiv preprint arXiv:2503.14476}, 2025.

\bibitem[Zeng et~al.(2025)Zeng, Zhang, Wu, Classen, Chae, Ewer, Lee, Kim, Kang, Kunde, et~al.]{zeng2025versaprm}
Zeng, T., Zhang, S., Wu, S., Classen, C., Chae, D., Ewer, E., Lee, M., Kim, H., Kang, W., Kunde, J., et~al.
\newblock {VersaPRM}: Multi-domain process reward model via synthetic reasoning data.
\newblock \emph{arXiv preprint arXiv:2502.06737}, 2025.

\bibitem[Zhang et~al.(2025{\natexlab{a}})Zhang, Chen, Pan, Zhao, Panda, Li, and He]{zhang2025reasoning}
Zhang, A., Chen, Y., Pan, J., Zhao, C., Panda, A., Li, J., and He, H.
\newblock Reasoning models know when they're right: Probing hidden states for self-verification.
\newblock \emph{arXiv preprint arXiv:2504.05419}, 2025{\natexlab{a}}.

\bibitem[Zhang et~al.(2024{\natexlab{a}})Zhang, Huang, Zhou, Li, and Ouyang]{zhang2024accessing}
Zhang, D., Huang, X., Zhou, D., Li, Y., and Ouyang, W.
\newblock Accessing {GPT-4} level mathematical {Olympiad} solutions via {Monte Carlo} tree self-refine with {LLaMa-3 8B}.
\newblock \emph{arXiv preprint arXiv:2406.07394}, 2024{\natexlab{a}}.

\bibitem[Zhang et~al.(2024{\natexlab{b}})Zhang, Zhoubian, Hu, Yue, Dong, and Tang]{zhang2024rest}
Zhang, D., Zhoubian, S., Hu, Z., Yue, Y., Dong, Y., and Tang, J.
\newblock {ReST-MCTS*}: {LLM} self-training via process reward guided tree search.
\newblock In \emph{Advances in Neural Information Processing Systems}, pp.\  64735--64772, 2024{\natexlab{b}}.

\bibitem[Zhang et~al.(2024{\natexlab{c}})Zhang, Hosseini, Bansal, Kazemi, Kumar, and Agarwal]{zhang2024generative}
Zhang, L., Hosseini, A., Bansal, H., Kazemi, M., Kumar, A., and Agarwal, R.
\newblock Generative verifiers: Reward modeling as next-token prediction.
\newblock \emph{arXiv preprint arXiv:2408.15240}, 2024{\natexlab{c}}.

\bibitem[Zhang et~al.(2025{\natexlab{b}})Zhang, Emma, En, and Dong]{zhang2025rvllm}
Zhang, Y., Emma, S.~Y., En, A. L.~J., and Dong, J.~S.
\newblock {RvLLM}: {LLM} runtime verification with domain knowledge.
\newblock \emph{arXiv preprint arXiv:2505.18585}, 2025{\natexlab{b}}.

\bibitem[Zhang et~al.(2025{\natexlab{c}})Zhang, Zheng, Wu, Zhang, Lin, Yu, Liu, Zhou, and Lin]{prmlessons}
Zhang, Z., Zheng, C., Wu, Y., Zhang, B., Lin, R., Yu, B., Liu, D., Zhou, J., and Lin, J.
\newblock The lessons of developing process reward models in mathematical reasoning.
\newblock \emph{arXiv preprint arXiv:2501.07301}, 2025{\natexlab{c}}.

\bibitem[Zhao et~al.(2025)Zhao, Aggarwal, Saha, Celikyilmaz, Weston, and Kulikov]{zhao2025majority}
Zhao, W., Aggarwal, P., Saha, S., Celikyilmaz, A., Weston, J., and Kulikov, I.
\newblock The majority is not always right: {RL} training for solution aggregation.
\newblock \emph{arXiv preprint arXiv:2509.06870}, 2025.

\bibitem[Zheng et~al.(2025{\natexlab{a}})Zheng, Zhang, Zhang, Lin, Lu, Yu, Liu, Zhou, and Lin]{zheng2024processbench}
Zheng, C., Zhang, Z., Zhang, B., Lin, R., Lu, K., Yu, B., Liu, D., Zhou, J., and Lin, J.
\newblock {P}rocess{B}ench: Identifying process errors in mathematical reasoning.
\newblock In \emph{Annual Meeting of the Association for Computational Linguistics}, pp.\  1009--1024, 2025{\natexlab{a}}.

\bibitem[Zheng et~al.(2023)Zheng, Chiang, Sheng, Zhuang, Wu, Zhuang, Lin, Li, Li, Xing, et~al.]{zheng2023judging}
Zheng, L., Chiang, W.-L., Sheng, Y., Zhuang, S., Wu, Z., Zhuang, Y., Lin, Z., Li, Z., Li, D., Xing, E., et~al.
\newblock Judging {LLM-as-a-Judge} with {MT-Bench} and {Chatbot Arena}.
\newblock In \emph{Advances in Neural Information Processing Systems}, pp.\  46595--46623, 2023.

\bibitem[Zheng et~al.(2025{\natexlab{b}})Zheng, Zhang, Yu, Wang, Dai, Liu, Bao, Huang, Huang, and Yu]{zheng2025parallel}
Zheng, T., Zhang, H., Yu, W., Wang, X., Dai, R., Liu, R., Bao, H., Huang, C., Huang, H., and Yu, D.
\newblock {Parallel-R1}: Towards parallel thinking via reinforcement learning.
\newblock \emph{arXiv preprint arXiv:2509.07980}, 2025{\natexlab{b}}.

\bibitem[Zheng et~al.(2024)Zheng, Zhang, Zhang, Ye, Luo, Feng, and Ma]{zheng2024llamafactory}
Zheng, Y., Zhang, R., Zhang, J., Ye, Y., Luo, Z., Feng, Z., and Ma, Y.
\newblock {LlamaFactory}: Unified efficient fine-tuning of 100+ language models.
\newblock \emph{arXiv preprint arXiv:2403.13372}, 2024.

\bibitem[Zhong et~al.(2025)Zhong, Li, Xu, Wen, Li, and Xu]{zhong2025solvedetectverifyinferencetimescalingflexible}
Zhong, J., Li, Z., Xu, Z., Wen, X., Li, K., and Xu, Q.
\newblock {Solve-Detect-Verify}: Inference-time scaling with flexible generative verifier.
\newblock \emph{arXiv preprint arXiv:2505.11966}, 2025.

\end{thebibliography}

\bibliographystyle{icml2026}

\newpage
\appendix
\onecolumn

\section{Algorithm Descriptions}\label{appendix:algorithm}
Algorithm~\ref{alg:otv-train-parallel} summarizes the parallelized OTV training procedure: for each sampled reasoning trace, we 1) cache the KV states produced by the frozen reasoner, 2) construct token-level pseudo-confidence targets, and 3) run a single parallel verification pass by inserting the truth token \texttt{[ToT]} at all probe positions. Algorithm~\ref{alg:otv-infer} describes the OTV inference procedure: a single verification token \texttt{[ToT]} probes any cached prefix and returns a token-level correctness estimate.
We use light purple shading to denote reasoning-mode steps (\ie, trace generation and KV caching) and light yellow shading to denote verification-mode steps (\ie, \texttt{[ToT]} probing).

\begin{algorithm}[h]
\setlength{\fboxsep}{4pt}
\caption{Parallelized OTV training}
\label{alg:otv-train-parallel}
\begin{algorithmic}[1]
\REQUIRE Dataset $\mathcal{D}$
    \STATE \algshade{reasonbg}{\textbf{Trace sampling + KV caching.} For each $(\bm q,a)\in \mathcal{D}$, sample a reasoning trace $\bm x_{1:T}\sim \mathrm{LLM}(\bm q)$ and cache per-layer KV states $\mathcal{C}_{T}$, where $\mathcal{C}_{t}=\{\K^{(l)}_{t},\V^{(l)}_{t}\}_{l=1}^{L}$}
    \STATE \algshade{reasonbg}{\textbf{Pseudo-confidence targets.} Compute token-level targets $\{c_t\}_{t=1}^{T}$ from the trace-level outcome label $y\in\{0,1\}$, where $y$ is derived by comparing the model's final answer to the ground-truth $a$}
    \STATE \algshade{verifybg}{\textbf{Single-pass parallel probing.} Construct a probe sequence $\texttt{[ToT]}_{1:T+1}$ and apply a triangular probe mask $\mathbf M$ so probe $t+1$ attends only to the cached prefix $1:t$; compute all token-level predictions $\{\hat c_t\}$ in one forward pass}
    \STATE \algshade{verifybg}{\textbf{Optimization.} Minimize the MSE over response tokens and update LoRA and regression head parameters}
\end{algorithmic}
\end{algorithm}
This implementation is equivalent to issuing $T+1$ separate one-token verification queries, but executes them in a single forward pass while reusing the same cached prefix states.

\begin{algorithm}[h]
\setlength{\fboxsep}{4pt}
\caption{OTV inference (one-token probing)}
\label{alg:otv-infer}
\begin{algorithmic}[1]
\REQUIRE Prompt/question $\bm q$, partial trace prefix $\bm x_{1:{t}}$, cached KV states $\mathcal{C}_{t}$
\STATE \algshade{reasonbg}{\textbf{Prefix generation + caching.} Generate (or reuse) a partial trace $\bm x_{1:t}$ and KV cache $\mathcal{C}_{t}$}
\STATE \algshade{verifybg}{\textbf{One-token verification query.} Insert a single \texttt{[ToT]} at position $t+1$, reuse $\mathcal{C}_{t}$, and run a forward pass under the LoRA-gated verifier to obtain $\hat{c}_t \in [0,1]$}  
\STATE \textbf{Trace-level scoring (optional).} Aggregate $\{\hat{c}_t\}$ (\eg, over the final segment) into a single score for ranking and pruning
\end{algorithmic}
\end{algorithm}

\section{Additional Experimental Details}

\subsection{Alternative Pseudo-Confidence Labeling Rules}\label{app:labeling_schemes}
OTV uses outcome supervision $y\in\{0,1\}$ to define dense token-level targets $c_t\in[0,1]$. Besides the default \emph{linear ramp}, we evaluate four alternatives.

\begin{itemize}[leftmargin=20pt]
    \item \emph{Constant label} uses a hindsight-style target that assumes the model ``knows'' the final outcome  across the trace:
    \begin{equation}
        c_t = y.
    \end{equation}

    \item \emph{Sigmoid ramp} provides a smooth, parameterized alternative that can interpolate between gradual and abrupt transitions:
    \begin{equation}
        c_t =\mathrm{sigmoid}\left((2y-1)\alpha\frac{t}{T}\right),
    \end{equation}
    where $\alpha >0$ controls the shape: smaller $\alpha$ yields an almost-linear ramp, while larger $\alpha$ produces a sharper transition toward the final label, approaching a step-like change (\ie, constant label). We set $\alpha = 4$ by default.

    \item \emph{Noise-perturbed ramp}. To model local fluctuations in uncertainty while preserving a global monotonic trend, we perturb the linear baseline with additive sinusoid and Gaussian noise:
    \begin{equation}
        c_t = \text{clip}\left( \mathrm{linear}(t) + \beta \sin\left(\frac{2\pi \omega t}{T}\right) + \epsilon, \; 0, 1 \right),
    \end{equation}
    where $\beta$ is the amplitude, $\omega$ is the frequency, and $\epsilon \sim \mathcal{N}(0, \sigma^2)$. We set $\beta = 0.1$ and $\omega = 3$ by default.

    \item \emph{Stepwise ramp} (by reasoning steps). Some traces are naturally organized into discrete reasoning steps. We segment the trace into $\Gamma$ steps using ``\textbackslash n\textbackslash n'' delimiters and let $\gamma(t)$ denote the index of the reasoning step to which token $t$ belongs. We then define
    \begin{equation}
        c_t = 0.5 + (y - 0.5) \frac{\gamma(t)}{\Gamma},
    \end{equation}
    giving rise to a piecewise-constant trajectory that increases only when a new step begins.
\end{itemize}

\subsection{Evaluation on Pretrained Base Models}\label{app:base_model}
\begin{wrapfigure}{r}{0.25\textwidth}
\begin{center}
\vspace{-0.6in}
\includegraphics[width=0.25\textwidth]{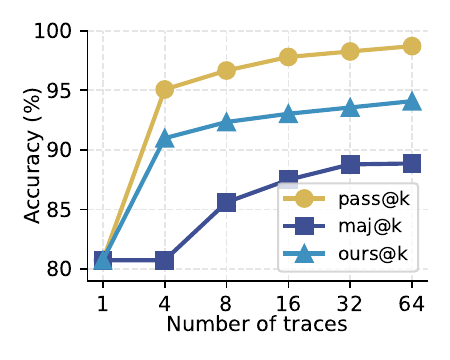}
\vspace{-0.3in}
\caption{Evaluation on GSM8K using \textit{Qwen3-4B-Base}.}
\vspace{-0.2in}
\label{fig:gsm8k}
\end{center}
\end{wrapfigure}
To further probe whether the observed gains are specific to post-trained reasoning models, we additionally evaluate OTV on a pretrained base model, \textit{Qwen3-4B-Base}. In this setting, we train on  MetaMathQA~\citep{yu2023metamath} and evaluate on GSM8K~\citep{cobbe2021training} to assess gains in raw mathematical reasoning capabilities. As shown in Figure~\ref{fig:gsm8k}, OTV substantially improves over Maj@$k$ and narrows the gap toward Pass@$k$, indicating that the verification signal remains effective even without instruction tuning. 

\begin{table*}[!t]
\centering
\renewcommand{\arraystretch}{1.0}
\caption{Weighted majority-voting accuracy on GSM8K using \textit{MetaMath-LLaMA-7B} and \textit{MetaMath-Mistral-7B} as backbone models. Results are reported as mean $\pm$ standard deviation over $64$ runs.}
\vspace{-0.1cm}
\begin{small}
\begin{sc}
\resizebox{\textwidth}{!}{
\begin{tabular*}{\textwidth}{@{\extracolsep{\fill}}clcccccc}
\toprule
 & $N$ & 1 & 4 & 8 & 16 & 32 & 64   \\
\midrule 
\multirow{5}{*}{\rotatebox{90}{\textit{LLaMA-7B}}} & Pass@$k$ & 61.87\small{$\pm$0.74}
& 75.01\small{$\pm$0.79} 
& 81.07\small{$\pm$0.49}
& 85.36\small{$\pm$0.48} 
& 88.68\small{$\pm$0.35}
& 91.27\small{$\pm$0.38} 
\\
\cmidrule(lr){2-8}
& Maj@$k$ & 61.87\small{$\pm$0.74}
& 67.29\small{$\pm$0.88} 
& 69.57\small{$\pm$0.43}
& 70.53\small{$\pm$0.44} 
& 71.47\small{$\pm$0.30}
& 71.79\small{$\pm$0.30}
\\
& DeepConf & 61.87\small{$\pm$0.74} & 68.64\small{$\pm$0.80} 
& 70.63\small{$\pm$0.60} 
& 71.76\small{$\pm$0.58}
& 72.06\small{$\pm$0.20} 
& 72.32\small{$\pm$0.21}
\\
& OTV ($r=2$) & 61.87\small{$\pm$0.74} & 
\textbf{71.13}\small{$\pm$0.74} & 
\textbf{73.50}\small{$\pm$0.47} & 
\textbf{74.55}\small{$\pm$0.52} & 
\textbf{74.80}\small{$\pm$0.37} & 
\textbf{75.16}\small{$\pm$0.22} \\
& OTV ($r=16$) & 61.87\small{$\pm$0.74}
& \underline{70.72}\small{$\pm$0.81} 
& \underline{72.82}\small{$\pm$0.35}
& \underline{73.80}\small{$\pm$0.32} 
& \underline{74.37}\small{$\pm$0.28}
& \underline{74.83}\small{$\pm$0.29} 
\\
\midrule
\multirow{5}{*}{\rotatebox{90}{\textit{Mistral-7B}}} & Pass@$k$ & 66.02\small{$\pm$1.18}
& 83.16\small{$\pm$0.65} 
& 88.46\small{$\pm$0.50} 
& 91.58\small{$\pm$0.44}
& 93.99\small{$\pm$0.28} 
& 95.87\small{$\pm$0.24}
\\
\cmidrule(lr){2-8}
& Maj@$k$ & 66.02\small{$\pm$1.18}
& 72.60\small{$\pm$0.49} 
& 77.33\small{$\pm$0.30} 
& 79.04\small{$\pm$0.37}
& 80.21\small{$\pm$0.40} 
& 80.56\small{$\pm$0.33}
\\
& DeepConf & 66.02\small{$\pm$1.18}
& 75.63\small{$\pm$0.90} 
& 78.62\small{$\pm$0.83} 
& 79.61\small{$\pm$0.34}
& 80.81\small{$\pm$0.37} 
& 81.31\small{$\pm$0.13}
\\
& OTV ($r=2$) & 66.02\small{$\pm$1.18}
& \textbf{79.52}\small{$\pm$0.73} 
& \underline{82.05}\small{$\pm$0.44} 
& \underline{83.12}\small{$\pm$0.25}
& \underline{84.06}\small{$\pm$0.24} 
& \underline{84.24}\small{$\pm$0.30}\\
& OTV ($r=16$) & 66.02\small{$\pm$1.18}
& \underline{79.51}\small{$\pm$0.57} 
& \textbf{82.27}\small{$\pm$0.30} 
& \textbf{83.54}\small{$\pm$0.42}
& \textbf{84.18}\small{$\pm$0.31} 
& \textbf{84.48}\small{$\pm$0.21}\\
\bottomrule
\end{tabular*}
}
\end{sc}
\end{small}
\label{tab:other}
\end{table*}

\subsection{Evaluation on Other Model Families}
\label{app:other_models}
To assess architectural generality beyond the \textit{Qwen} family, OTV is applied to  \textit{MetaMath-LLaMA-7B}\footnote{\url{https://huggingface.co/meta-math/MetaMath-7B-V1.0}} and  \textit{MetaMath-Mistral-7B}\footnote{\url{https://huggingface.co/meta-math/MetaMath-Mistral-7B}} models on GSM8K under low-resource fine-tuning (using only 10k MetaMathQA subset). As shown in Table~\ref{tab:other},  OTV consistently outperforms majority voting and DeepConf across sample sizes, and increasing LoRA rank typically improves performance.

\subsection{Inference-time Verification Overhead}\label{sec:output-length}
OTV's verification calls depend on the aggregation strategy (illustrated for $N=128$).
\begin{itemize}[leftmargin=15pt]
    \item \textbf{Drop@10} ($8,256$ calls)  prunes the \textit{single} lowest-scoring trace every $10$ generated tokens until only one trace remains. At the first checkpoint ($t=10$), all $128$ traces are verified and the worst is discarded; at the second ($t=20$), the remaining $127$ traces are verified and one is dropped, and so on. The total number of verification calls is the arithmetic sum:
    \begin{equation}
        \#\text{Calls} = \sum_{n=1}^{N} n = \frac{N(N+1)}{2} = \frac{128 \times 129}{2} = 8,256.
    \end{equation}

    \item \textbf{Stop@600} ($128$ calls) as a ``verify-once'' strategy,  triggers the verifier only when a trace reaches the $600$-th token (or earlier if generation terminates). Consequently, each of the $N$ candidate traces is verified exactly once:
    \begin{equation}
        \#\text{Calls} = N = 128.
    \end{equation}

    \item \textbf{Halve@300} ($254$ calls) halves the candidate pool every $300$ tokens. That is, we verify $128$ traces at $t=300$, $64$ at $t=600$, $32$ at $t=900$, and so forth. The total number of calls follows a geometric series with the upper bound:
    \begin{equation}
        \#\text{Calls} \approx \sum_{k=0}^{
        \left\lceil\log_2(N -1)\right\rceil} \left\lceil\frac{N}{2^k}\right\rceil = 128 + 64 + 32 + \dots + 4 + 2 = 254.  
    \end{equation}
\end{itemize}

\section{Theoretical Analysis}\label{app:theory}

This section formalizes what OTV learns under different pseudo-confidence labeling rules by viewing verification as square-loss regression on a trace prefix $\bm s_{t} = [\bm q,\bm x_{1:t}]$, equivalently, its KV cache.

\subsection{Setups}
Let the base reasoner induce an autoregressive policy $\pi_{\bm\theta}(\bm x_t \mid \bm q, \bm x_{1:t})$ and sample traces $\bm x_{1:T} = [\bm x_1,\ldots,\bm x_T]$ with terminal correctness $y\in\{0,1\}$. For any pseudo-confidence labeling rule $c_t = c(t, T, y)$, OTV minimizes the per-token MSE:
\begin{equation}
\ell(\bm \phi)=\mathbb{E}\!\left[\left(f_{\bm \phi}(\bm s_{t})-c(t,T,y)\right)^2\right],
\end{equation}
where $f_{\bm \phi}(\cdot)$ denotes the verifier, parameterized by vector $\bm \phi$ and the expectation is taken over trajectories  $\bm x_{1:T}\sim \pi_{\bm\theta}(\cdot\mid \bm q)$ sampled from the base reasoner.

\subsection{Optimal predictor under MSE}
\begin{proposition}[Risk Minimizer under MSE]\label{prop:mse_opt}
For any fixed $t$ and any state $\bm s_{t}$, among all measurable functions $f_
{\bm \phi}(\cdot)$,
the minimizer of the conditional risk
$\mathbb{E}\!\left[(f_
{\bm \phi}(\bm s_{t})-c(t,T,y))^2 \mid \bm s_{t}\right]$
is the conditional expectation
\begin{equation}\label{eq:conditional_expectation_opt}
f_{\bm \phi^\star}(\bm s_{t})
\;=\;
\mathbb{E}\!\left[c(t,T,y)\mid \bm s_{t}\right].
\end{equation}
\end{proposition}
\begin{proof}
Fix $\bm s_{t}$ and consider any scalar $a$. By the bias-variance decomposition,
\begin{equation}
\mathbb{E}\!\left[(a-c(t,T,y))^2\mid \bm s_{t}\right]
=
\left(a-\mathbb{E}[c(t,T,y)\mid \bm s_{t}]\right)^2 + \mathrm{Var}\left(c(t,T,y)\mid \bm s_{t}\right),
\end{equation}
which is minimized at $a=\mathbb{E}[c(t,T,y)\mid \bm s_{t}]$.
\end{proof}
This shows that the choice of pseudo-confidence labeling rule $c_t$ determines the conditional statistic that OTV would like to approximate during training.

\subsection{Constant Label Recovers a Monte-Carlo Correctness Value}
For the \textit{constant label} rule $c_\mathrm{const}(t,T,y) = y$, the optimal predictor equals
\begin{equation}
V_{\mathrm{MC}}(\bm s_{t})
\;:=\;
\mathbb{P}(y=1\mid \bm s_{t})
\;=\;
\mathbb{E}[y\mid \bm s_{t}],
\end{equation}
\ie, the probability that a continuation sampled from $\pi_{\bm\theta}(\cdot\mid \bm s_{t})$ yields a correct final answer.

\subsection{Linear Ramp Induces an Inverse-Length Preference}
Consider the \textit{linear ramp} $c_{\mathrm{lin}}(t,T,y)=0.5+(y-0.5)\frac{t}{T}$, which satisfies $c_{\mathrm{lin}}(t,T,y)\in[0,1]$ for $t\le T$, and $c_{\mathrm{lin}}(T,T,y)=y$. Define the expected inverse lengths of correct and incorrect completions from $\bm s_{t}$ as
\begin{equation}
\mu_{+}(\bm s_{t})
:=\mathbb{E}\!\left[\frac{1}{T}\,\middle|\, \bm s_{t},\, y=1\right],
\qquad
\mu_{-}(\bm s_{t})
:=\mathbb{E}\!\left[\frac{1}{T}\,\middle|\, \bm s_{t},\, y=0\right].
\end{equation}

\begin{proposition}[]\label{prop:linear_affine}
The MSE-optimal predictor under $c_\mathrm{lin}$ is
\begin{equation}\label{eq:linear_opt_closed_form}
f_{\bm \phi^\star}(\bm s_{t})
=
\underbrace{0.5 - \frac{t}{2}\mu_{-}(\bm s_{t})}_{\mathrm{baseline}}
\;+\;
V_{\mathrm{MC}}(\bm s_{t})\cdot
\underbrace{\frac{t}{2}\big(\mu_{+}(\bm s_{t})+\mu_{-}(\bm s_{t})\big)}_{\mathrm{gain}>0},
\end{equation}
where $V_{\mathrm{MC}}(\bm s_{t})
\;=\;
\mathbb{E}[y\mid \bm s_{t}]$. In particular, for fixed $\bm s_{t}$ and $t$, $f_{\bm \phi^\star}(\bm s_{t})$ is strictly increasing in
$V_{\mathrm{MC}}(\bm s_{t})$.
\end{proposition}

\begin{proof}
By Proposition~\ref{prop:mse_opt}
\begin{equation}
f_{\bm\phi^\star}(\bm s_{t})
=
\mathbb{E}\!\left[0.5 + (y-0.5)\frac{t}{T}\,\middle|\, \bm s_{t}\right]
=
0.5 + t\cdot \mathbb{E}\!\left[\frac{y-0.5}{T}\,\middle|\, \bm s_{t}\right].
\label{eq:lin_step1}
\end{equation}

Conditioning on $y\in\{0,1\}$ yields
\begin{align}
\mathbb{E}\!\left[\frac{y-0.5}{T}\,\middle|\, \bm s_{t}\right]
&=
\mathbb{P}(y=1\mid \bm s_{t})\cdot \mathbb{E}\!\left[\frac{0.5}{T}\,\middle|\, \bm s_{t},y=1\right]
+
\mathbb{P}(y=0\mid \bm s_{t})\cdot \mathbb{E}\!\left[\frac{-0.5}{T}\,\middle|\, \bm s_{t},y=0\right]\nonumber\\
&=
\frac{1}{2}V_{\mathrm{MC}}(\bm s_{t})\mu_{+}(\bm s_{t})
-\frac{1}{2}(1-V_{\mathrm{MC}}(\bm s_{t}))\mu_{-}(\bm s_{t}), 
\label{eq:lin_step2}
\end{align}
and substituting into Eq.~(\ref{eq:lin_step1}) gives
Eq.~(\ref{eq:linear_opt_closed_form}). Since $\mu_{+},\mu_{-}>0$, the coefficient on $V_{\mathrm{MC}}$ equals $\frac{t}{2}(\mu_{+}+\mu_{-})>0$, implying strict monotonicity.
\end{proof}

Eq.~(\ref{eq:linear_opt_closed_form}) shows that, beyond its monotonic dependence on $V_{\mathrm{MC}}$, the score is modulated by
$\mu_{+}$ and $\mu_{-}$; holding $V_{\mathrm{MC}}$ fixed, larger
$\mu_{+}$ (\ie, shorter correct completions in expectation) increases
$f_{\bm \phi^\star}$, inducing a preference on shorter correct traces.

\section{Future Work}\label{app:future}

Building on these observations, several directions appear promising. The first is to tighten model-verifier co-adaptation. Rather than treating the verifier as a fixed add-on, we may explore joint or continual training where the verifier tracks changes in the base model and the evolving decoding distribution. This naturally connects \textit{model-centric} improvements (\eg, stronger reasoning backbones, better calibration, and decoding-time optimization) with \textit{data-centric} improvements: OTV's scores can be used to curate higher-quality training traces, perform hard-negative mining, and support active learning by prioritizing ``high-uncertainty'' cases for annotation or additional compute.

Second, improving token-level pseudo-confidence labeling is a key opportunity.  The current approach derives dense targets from final outcomes, but richer supervision could substantially improve performance. One direction is uncertainty-aware pseudo-labels that represent partial progress, reversible mistakes, or late-step slips rather than enforcing uniformly monotonic confidence trajectories. Another direction is hybrid bootstrapping for token-level supervision: augment outcome labels with auxiliary signals such as self-consistency/ensemble disagreement, step-boundary priors (\eg, transition markers between reasoning steps), and selective distillation from stronger but costlier process verifiers on a small subset. This can be implemented iteratively, where each improved verifier produces higher-fidelity pseudo-labels for subsequent training rounds, thereby reducing systematic bias introduced by weak initial targets. 

Third, it is desirable to extend the verifier output space and its role in decision-making, with the goal of unlocking broader applications. For example, a selective prediction formulation, \eg, moving from binary confidence to a ternary ``correct/incorrect/unknown'' signal, would allow for abstention and risk-controlled routing when the model is uncertain. Finally, it would be valuable to test OTV beyond math benchmarks, including code reasoning and tool-augmented tasks, and to study how verifier-guided compute allocation interacts with long-context settings (\eg, multi-step planning and multi-agent/ensemble decoding).

\section{Sensitivity to Solution Perturbations}
To test robustness to superficial edits vs. semantic errors, we construct controlled perturbations of the same solution text (see Table~\ref{tab:modification}). Semantics-preserving edits (\eg, shortening/lengthening, formatting changes, or substituting the ground-truth answer) are contrasted with logic-breaking edits (\eg, changing numbers, altering logic, or corrupting the final step). Empirically, we observe that OTV’s confidence reliably increases for correct variants and drops for erroneous variants, suggesting it is driven by internal reasoning consistency rather than surface form.

\begin{table}[h]
\centering
\caption{Representative perturbed solution variants used in the sensitivity analysis. Each row applies a controlled edit to the same base solution to evaluate how the verifier’s confidence responds to semantics-preserving vs. logic-breaking modifications.}
\begin{tabularx}{\linewidth}{@{}l L@{}}
\toprule
\textbf{Type} & \textbf{Text} \\
\midrule
\multicolumn{2}{c}{{Semantics-preserving modifications}} \\
\midrule
\tiny{Ground-truth} &
\tiny{It takes $2/2=1$ bolt of white fiber\textbackslash nSo the total amount of fabric is $2+1=3$ bolts of fabric\textbackslash n\#\#\#\# 3} \\
\addlinespace[2pt]
\tiny{Original} &
\tiny{To solve this problem, we need to determine the total number of bolts required for the robe, which includes both blue and white fiber.\textbackslash n\textbackslash nStep 1: Identify the amount of blue fiber needed.\textbackslash nThe robe takes 2 bolts of blue fiber.\textbackslash n\textbackslash nStep 2: Calculate the amount of white fiber needed.\textbackslash nThe robe takes half as much white fiber as blue fiber. Since it takes 2 bolts of blue fiber, it will take 2 / 2 = 1 bolt of white fiber.\textbackslash n\textbackslash nStep 3: Add the amounts of blue and white fiber to find the total number of bolts.\textbackslash nTotal bolts = Blue fiber + White fiber\textbackslash nTotal bolts = 2 + 1\textbackslash nTotal bolts = 3.} \\
\addlinespace[2pt]
\tiny{Shortening} &
\tiny{To find total bolts for the robe (blue + white fiber):\textbackslash n\textbackslash nStep 1: Blue fiber needed is 2 bolts.\textbackslash n\textbackslash nStep 2: White fiber is half blue, so 2/2=1 bolt.\textbackslash n\textbackslash nStep 3: Total = 2+1=3.} \\
\addlinespace[2pt]
\tiny{Lengthening} &
\tiny{To solve the problem of determining the total number of bolts required for the robe, we must consider both the blue fiber and white fiber that are necessary for its construction.\textbackslash n\textbackslash nStep 1: First, we need to identify the specific amount of blue fiber that is needed to make the robe. After checking the relevant specifications or requirements for the robe, it is determined that the robe takes 2 bolts of blue fiber.\textbackslash n\textbackslash nStep 2: Next, we need to calculate the amount of white fiber that is needed. According to the given information, the robe takes half as much white fiber as blue fiber. Since we already established that it takes 2 bolts of blue fiber, we can find the amount of white fiber by taking half of that number. To do this, we perform the calculation: 2 bolts (blue fiber) divided by 2, which equals 1 bolt. Therefore, the robe requires 1 bolt of white fiber.\textbackslash n\textbackslash nStep 3: Finally, to find the total number of bolts required for the robe, we need to add together the amounts of blue fiber and white fiber that we have determined. This means we take the 2 bolts of blue fiber and add the 1 bolt of white fiber to it. Performing this addition gives us: 2 + 1 = 3.} \\
\addlinespace[2pt]
\tiny{Removing newline token} &
\tiny{To solve this problem, we need to determine the total number of bolts required for the robe, which includes both blue and white fiber. Step 1: Identify the amount of blue fiber needed. The robe takes 2 bolts of blue fiber. Step 2: Calculate the amount of white fiber needed. The robe takes half as much white fiber as blue fiber. Since it takes 2 bolts of blue fiber, it will take 2 / 2 = 1 bolt of white fiber. Step 3: Add the amounts of blue and white fiber to find the total number of bolts. Total bolts = Blue fiber + White fiber. Total bolts = 2 + 1. Total bolts = 3. Answer: 3.} \\
\addlinespace[2pt]
\midrule
\multicolumn{2}{c}{{Logic-breaking modifications}} \\
\midrule
\tiny{Changing numbers} &
\tiny{To solve this problem, we need to determine the total number of bolts required for the robe, which includes both blue and white fiber.\textbackslash n\textbackslash nStep 1: Identify the amount of blue fiber needed.\textbackslash nThe robe takes 4 bolts of blue fiber.\textbackslash n\textbackslash nStep 2: Calculate the amount of white fiber needed.\textbackslash nThe robe takes half as much white fiber as blue fiber. Since it takes 4 bolts of blue fiber, it will take 4 / 2 = 2 bolts of white fiber.\textbackslash n\textbackslash nStep 3: Add the amounts of blue and white fiber to find the total number of bolts.\textbackslash nTotal bolts = Blue fiber + White fiber\textbackslash nTotal bolts = 4 + 2\textbackslash nTotal bolts = 6\textbackslash n\textbackslash nAnswer: 6.} \\
\addlinespace[2pt]
\tiny{Altering logic} &
\tiny{To solve this problem, we need to determine the total number of bolts required for the robe, which includes both blue and white fiber.\textbackslash n\textbackslash nStep 1: Identify the amount of blue fiber needed.\textbackslash nThe robe takes 2 bolts of white fiber.\textbackslash n\textbackslash nStep 2: Calculate the amount of white fiber needed.\textbackslash nThe robe takes twice as much white fiber as blue fiber. Since it takes 2 bolts of blue fiber, it will take 2 * 2 = 4 bolts of white fiber.\textbackslash n\textbackslash nStep 3: Subtract the amounts of blue and white fiber to find the total number of bolts.\textbackslash nTotal bolts = Blue fiber - White fiber\textbackslash nTotal bolts = 2 - 1\textbackslash nTotal bolts = 1\textbackslash n\textbackslash nAnswer: 1.} \\
\addlinespace[2pt]
\tiny{Adding extra steps} &
\tiny{To solve this problem, we need to determine the total number of bolts required for the robe, which includes both blue and white fiber.\textbackslash n\textbackslash nStep 1: Identify the amount of blue fiber needed.\textbackslash nThe robe takes 2 bolts of blue fiber.\textbackslash n\textbackslash nStep 2: Calculate the amount of white fiber needed.\textbackslash nThe robe takes half as much white fiber as blue fiber. Since it takes 2 bolts of blue fiber, it will take 2 / 2 = 1 bolt of white fiber.\textbackslash n\textbackslash nStep 3: Add the amounts of blue and white fiber to find the total number of bolts.\textbackslash nTotal bolts = Blue fiber + White fiber\textbackslash nTotal bolts = 2 + 1\textbackslash nTotal bolts = 3\textbackslash n\textbackslash nStep 4: Subtract 1 from the total because of a miscalculation.\textbackslash nTotal bolts = 3 - 1\textbackslash nTotal bolts = 2\textbackslash n\textbackslash nAnswer: 2.} \\
\addlinespace[2pt]
\tiny{Corrupting final steps} &
\tiny{To solve this problem, we need to determine the total number of bolts required for the robe, which includes both blue and white fiber.\textbackslash n\textbackslash nStep 1: Identify the amount of blue fiber needed.\textbackslash nThe robe takes 2 bolts of blue fiber.\textbackslash n\textbackslash nStep 2: Calculate the amount of white fiber needed.\textbackslash nThe robe takes half as much white fiber as blue fiber. Since it takes 2 bolts of blue fiber, it will take 2 / 2 = 1 bolt of white fiber.\textbackslash n\textbackslash nStep 3: Multiply the amounts of blue and white fiber to find the total number of bolts.\textbackslash nTotal bolts = Blue fiber × White fiber\textbackslash nTotal bolts = 2 × 1\textbackslash nTotal bolts = 2\textbackslash n\textbackslash nAnswer: 2.} \\
\addlinespace[2pt]
\tiny{Repeating the question} &
\tiny{To solve this problem, we need to determine the total number of bolts required for the robe, which includes both blue and white fiber. A robe takes 2 bolts of blue fiber and half that much white fiber.  How many bolts in total does it take? A robe takes 2 bolts of blue fiber and half that much white fiber.  How many bolts in total does it take? A robe takes 2 bolts of blue fiber and half that much white fiber.  How many bolts in total does it take? A robe takes 2 bolts of blue fiber and half that much white fiber.  How many bolts in total does it take? A robe takes 2 bolts of blue fiber and half that much white fiber.  How many bolts in total does it take?}\\
\addlinespace[2pt]
\bottomrule
\end{tabularx}
\label{tab:modification}
\end{table}

\newpage

\section{Trace-Level Confidence Dynamics on AIME Problems}\label{appendix:morevis}
Figure~\ref{fig:aime_conf} plots trace-level confidence trajectories for all AIME24/25 problems under \textit{Qwen3-4B-Instruct}. Each curve corresponds to one sampled trace, with correct traces in red and incorrect traces in green.

\begin{figure}[!h]
    \centering
    \begin{subfigure}{1.0\textwidth}
        \centering
        \includegraphics[width=0.85\linewidth]{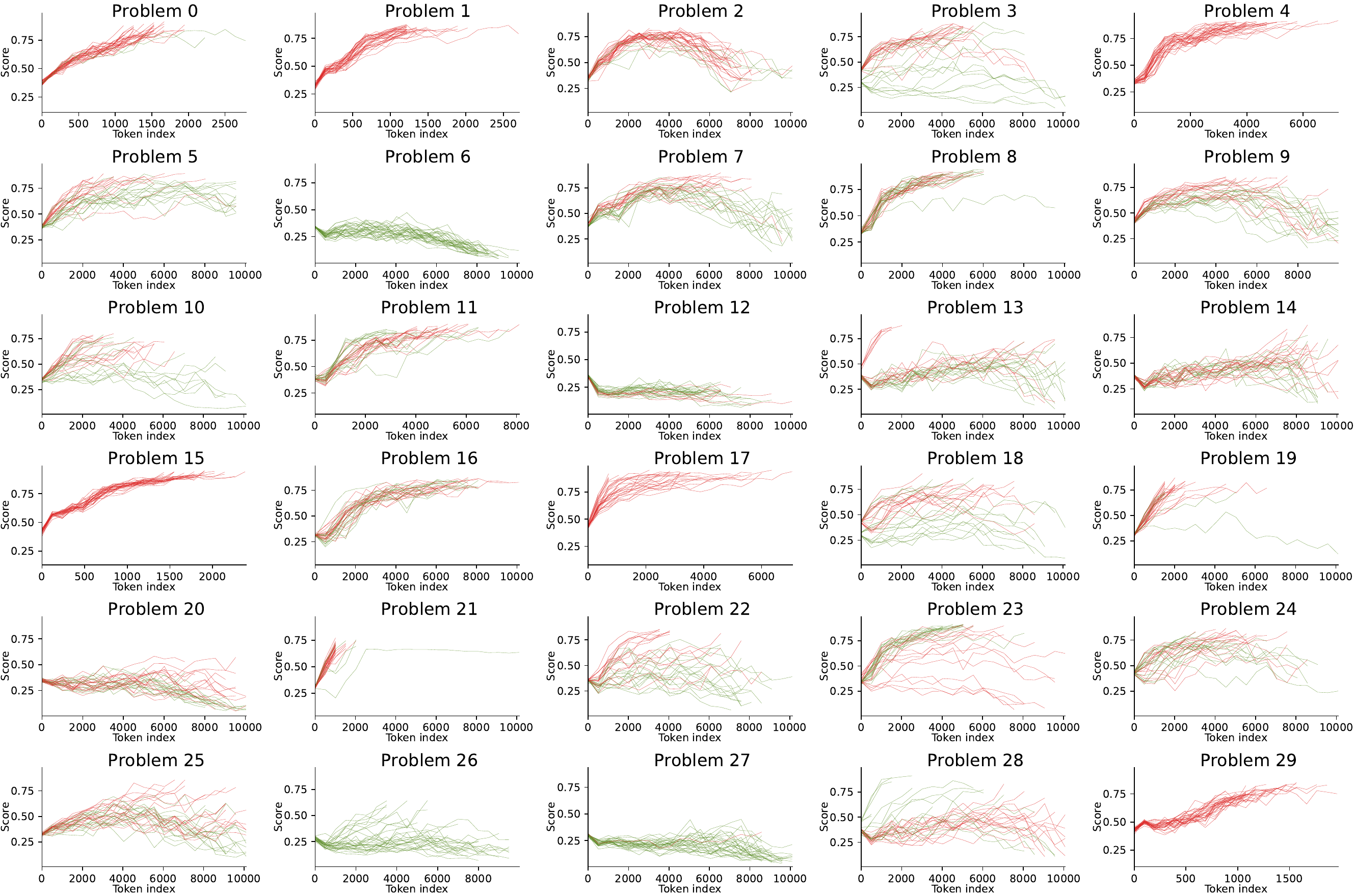}
        \vspace{-0.05in}
        \caption{AIME24}
        \label{fig:aime24_conf}
    \end{subfigure}
    \vspace{0.02in}
    \begin{subfigure}{1.0\textwidth}
        \centering
        \includegraphics[width=0.85\linewidth]{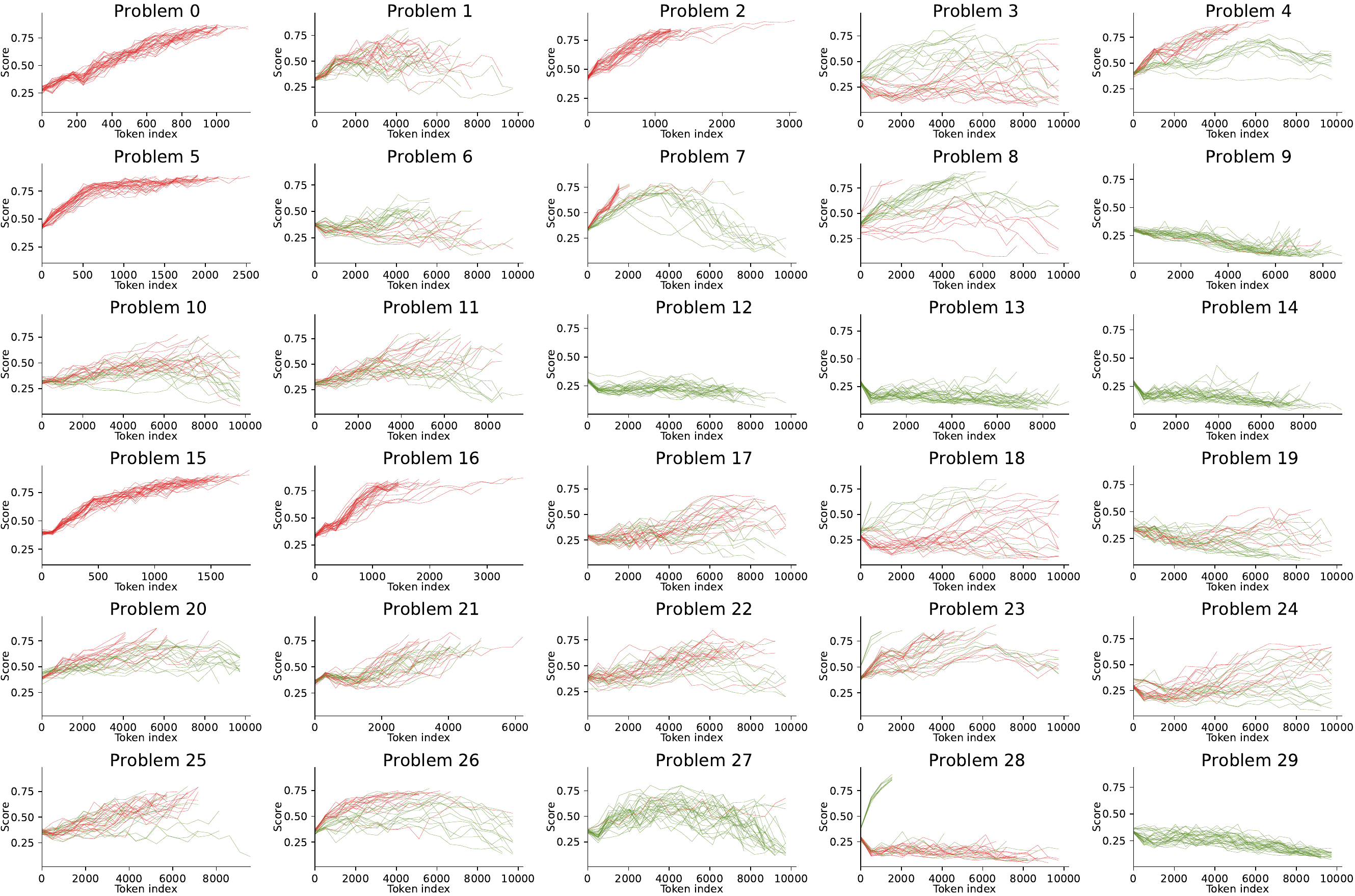}
        \vspace{-0.05in}
        \caption{AIME25}
        \label{fig:aime25_conf}
    \end{subfigure}
    \caption{Trace-level confidence trajectories across problems in (a) AIME24 and (b) AIME25.}
    \label{fig:aime_conf}
\end{figure}



\section*{Supplementary material}

## G. Finer-Grained Visualization of Reasoning Correctness Estimation

To complement the main results, we present token-level visualizations of OTV's correctness estimates, where the generated reasoning text is color-coded from green (low confidence) to red (high confidence). We include six examples spanning three comparisons: 1) incorrect vs. correct traces for the same problem (Examples 1–2), demonstrating clear confidence separation; 2) shorter vs. longer correct solutions (Examples 3–4), where shorter traces sustain consistently higher confidence while longer traces show greater fluctuations, indicating intermediate uncertainty; and 3) correct vs. incorrect traces produced by a larger model on the same problem (Examples 5–6), where confidence declines near the end in both cases, but remains substantially higher overall for correct traces.

---

**Example 1 (AIME25, Problem 23): Qwen3-4B-Instruct-OTV**

There are $n$ values of $x$ in the interval $0 < x < 2\pi$ where $f(x) = \sin(7\pi \cdot \sin(5x)) = 0$. For $t$ of these $n$ values of $x$, the graph of $y = f(x)$ is tangent to the $x$-axis. Find $n + t$. (Correct Answer: 149)

- - - - - - - - - - - - - - - - - - - - - - - - - - - - - - - - - - - - - - - - - - - - - - - - - - - - - - - -

We are given the function: $f(x) = \sin(7\pi \cdot \sin(5x))$ We are to find: - The total number $n$ of values of $x$ in the interval $0 x 2\pi$ such that $f(x) = 0$ - Among these $n$ values, how many $t$ are such that the graph is **tangent** to the $x$-axis (i.e., $f(x) = 0$ and $f'(x) = 0$) - Then compute $n + t$ — Step 1: When is $f(x) = 0$? We have: $f(x) = \sin(7\pi \cdot \sin(5x)) = 0$ The sine function is zero when its argument is an integer multiple of $\pi$: $7\pi \cdot \sin(5x) = k\pi$ for some integer $k$ Divide both sides by $\pi$: $7\sin(5x) = k \Rightarrow \sin(5x) = \frac{k}{7}$ Now, since $\sin(5x) \in [-1, 1]$, we must have: $-1 \le \frac{k}{7} \le 1 \Rightarrow -7 \le k \le 7$ So $k$ ranges from $-7$ to 7, inclusive. That's 15 integer values. For each such $k$, we solve: $\sin(5x) = \frac{k}{7}$ We want to find how many **solutions** $x \in (0, 2\pi)$ satisfy this equation. Let's define $\theta = 5x$. Then $x \in (0, 10\pi)$ So we are solving $\sin(\theta) = \frac{k}{7}$ for $\theta \in (0, 10\pi)$ — Step 2: Number of solutions to $\sin(\theta) = c$ in $(0, 10\pi)$ For a fixed $c \in (-1, 1)$, the equation $\sin(\theta) = c$ has **two solutions per period** of $2\pi$, except when $c = \pm 1$, where it has one solution per period (at the peak/trough). So we analyze based on the value of $k$: Case 1: $k = 0$ Then $\sin(5x) = 0$ Solutions: $\sin(\theta) = 0$, $\theta \in (0, 10\pi)$ Solutions occur at $\theta = \pi, 2\pi, 3\pi, 4\pi, 5\pi, 6\pi, 7\pi, 8\pi, 9\pi$ Note: $\theta = 0$ and $10\pi$ are excluded because $x = 0$ and $x = 2\pi$ are excluded. So $\theta = \pi, 2\pi, 3\pi, 4\pi, 5\pi, 6\pi, 7\pi, 8\pi, 9\pi \Rightarrow 9$ solutions Each corresponds to $x = \theta/5$, so $x = \pi/5, 2\pi/5, \dots, 9\pi/5$ All are in $(0, 2\pi)$ So **9 solutions** for $k = 0$ — Case 2: $k \ne 0$ We consider $k = \pm 1, \pm 2, \dots, \pm 7$ For each such $k$, $\frac{k}{7} \in (-1, 1)$, so $\sin(\theta) = \frac{k}{7}$ has **two solutions per period** of $2\pi$ in $(0, 10\pi)$ How many periods in $(0, 10\pi)$? $\frac{10\pi}{2\pi} = 5$ full periods So in each period, two solutions $\to$ total of $5 \times 2 = 10$ solutions per value of $k \in \{-7, \dots, 7\} \setminus \{0\}$ But wait — are there any exceptions? Only when $\sin(\theta) = \pm 1$, i.e., $k = \pm 7$ Let's check: - $k = 7$: $\sin(5x) = 1$ - $k = -7$: $\sin(5x) = -1$ For these, $\sin(\theta) = \pm 1$ occurs once: at $\theta = \frac{\pi}{2} + 2\pi m$ Similarly, $\sin(\theta) = -1$ at $\theta = \frac{3\pi}{2} + 2\pi m$ So in $(0, 10\pi)$, how many such solutions? - For $\sin(\theta) = 1$: - For $\sin(\theta) = 1$: $\theta = \frac{\pi}{2}, \frac{5\pi}{2}, \frac{9\pi}{2}, \frac{13\pi}{2}, \frac{17\pi}{2}$ Check: $\frac{\pi}{2} + 2\pi m 10\pi$ Divide by $\pi$: $\frac{1}{2} + 2m 10 \Rightarrow 2m 9.5 \Rightarrow m 4.75$ So $m = 0, 1, 2, 3, 4 \to 5$ values Similarly for $\sin(\theta) = -1$: $\theta = \frac{3\pi}{2} + 2\pi m$ $\frac{3}{2} + 2m 10 \Rightarrow 2m 9.5 \Rightarrow m 4.75 \to m = 0, 1, 2, 3, 4 \to 5$ values So for $k = \pm 7$, we have **5 solutions each** For other $k = \pm 1, \pm 2, \pm 3, \pm 4, \pm 5, \pm 6$, we have $\sin(\theta) = \frac{k}{7} \in (-1, 1)$, so **two solutions per period**, 5 periods $\to 5 \times 2 = 10$ solutions each — Step 3: Count total number of solutions $n$ List all $k$ from $-7$ to 7: - $k = 0$: 9 solutions - $k = \pm 1, \pm 2, \pm 3, \pm 4, \pm 5, \pm 6$: 6 values of $k$, each with 10 solutions $\to 6 \times 10 = 60$ - $k = \pm 7$: each with 5 solutions $\to 2 \times 5 = 10$ Total $n = 9 + 60 + 10 = 79$ So $n = 79$ — Step 4: Find $t$ — number of values where the graph is **tangent** to the $x$-axis This means $f(x) = 0$ **and** $f'(x) = 0$ We need to find the values of $x$ where $f(x) = 0$ and $f'(x) = 0$ First, recall: $f(x) = \sin(7\pi \cdot \sin(5x))$ Let $u(x) = 7\pi \cdot \sin(5x)$ So $f(x) = \sin(7\pi \cdot \sin(5x))$ Let $u(x) = 7\pi \cdot \sin(u(x))$ Then: $f'(x) = \cos(u(x)) \cdot u'(x)$ Compute $u'(x)$: $u'(x) = 7\pi \cdot \cos(5x) \cdot 5 = 35\pi \cos(5x)$ So: $f'(x) = \cos(7\pi \cdot \sin(5x)) \cdot 35\pi \cos(5x)$ We want $f'(x) = 0$ **and** $f(x) = 0$ So two conditions: 1. $\sin(7\pi \sin(5x)) = 0$ 2. $\cos(7\pi \sin(5x)) \cdot 35\pi \cos(5x) = 0$ Note that $35\pi \ne 0$, so we can divide: So condition 2 becomes: $\cos(7\pi \sin(5x)) = 0$ or $\cos(5x) = 0$ But from condition 1: $\sin(7\pi \sin(5x)) = 0$ Let $s = \sin(5x)$. Then: - $\sin(7\pi s) = 0$ - Either $\cos(7\pi s) = 0$ or $\cos(5x) = 0$ We are to find values where **both** conditions hold. So we consider two cases: — Case A: $\cos(5x) = 0$ Then $5x = \frac{\pi}{2} + m\pi \Rightarrow x = \frac{\pi}{10} + \frac{m\pi}{5}$ We need $x \in (0, 2\pi)$ So $0 \frac{1}{10} + \frac{m}{5} 2\pi$ Divide by $\pi$: $0 \frac{1}{10} + \frac{m}{5} 2$ Multiply by 10: $0 1 + 2m 20 2m - 1 \Rightarrow m - 0.5 \Rightarrow m \ge 0 2m 19 \Rightarrow m 9.5 \Rightarrow m \le 9$ So $m = 0, 1, 2, \dots, 9 \to 10$ values Now, for each such $x$, we must check whether $f(x) = 0$, i.e., $\sin(7\pi \sin(5x)) = 0$ At such $x$, $\cos(5x) = 0$, so $5x = \frac{\pi}{2} + m\pi$ Then $\sin(5x) = \sin(\frac{\pi}{2} + m\pi) = (-1)^m$ So $\sin(5x) = (-1)^m \in \{-1, 1\}$ Then $7\pi \sin(5x) = \pm 7\pi$ So $\sin(7\pi \sin(5x)) = \sin(\pm 7\pi) = \sin(\mp 7\pi)$ Because $\sin(7\pi) = \sin(\pi) = 0$ So yes, $f(x) = 0$ Also, $f'(x) = 0$ because $\cos(5x) = 0$ So **all 10 values** of $x$ satisfy both $f(x) = 0$ and $f'(x) = 0$ So **$t$ includes these 10 values** — Case B: $\cos(7\pi \sin(5x)) = 0$ **and** $\sin(7\pi \sin(5x)) = 0$ Is this possible? Let $\theta = 7\pi \sin(5x)$ We want: $\sin(\theta) = 0 \cdot \cos(\theta) = 0$ But $\sin(\theta) = 0$ and $\cos(\theta) = 0$ is **impossible** for any real $\theta$ Because $\sin^2(\theta) + \cos^2(\theta) = 1$ So both cannot be zero. Therefore, **no solutions** in this case. So the only points where $f(x) = 0$ and $f'(x) = 0$ are when $\cos(5x) = 0$ Thus, $t = 10$ — Final Step: Compute $n + t$ We found: - $n = 79$ - $t = 10$ So: Answer: 89

19

**Example 2 (AIME25, Problem 23): Qwen3-4B-Instruct-OTV**

There are $n$ values of $x$ in the interval $0 < x < 2\pi$ where $f(x) = \sin(7\pi \cdot \sin(5x)) = 0$. For $t$ of these $n$ values of $x$, the graph of $y = f(x)$ is tangent to the $x$-axis. Find $n + t$. (Correct Answer: 149)

---

We are given the function: $f(x) = \sin(7\pi \cdot \sin(5x))$ We are to find: - $n$: the number of values of $x$ in the interval $0 < x < 2\pi$ where $f(x) = 0$ - $t$: the number of those $n$ values where the graph is **tangent** to the x-axis (i.e., where $f(x) = 0$ and $f'(x) = 0$) - Then compute $n + t$ — Step 1: When is $f(x) = 0$? We have: $f(x) = \sin(7\pi \cdot \sin(5x)) = 0$ The sine function is zero when its argument is an integer multiple of $\pi$: $7\pi \cdot \sin(5x) = k\pi$ for some integer $k$ Divide both sides by $\pi$: $7\sin(5x) = k \implies \sin(5x) = \frac{k}{7}$ Now, since $\sin(5x) \in [-1, 1]$, we require: $-1 \le \frac{k}{7} \le 1 \implies -7 \le k \le 7$ So $k = -7, -6, \ldots, 0, \ldots, 6, 7$ That's 15 integer values of $k$. For each such $k$, we solve: $\sin(5x) = \frac{k}{7}$ We now want to count how many solutions $x \in (0, 2\pi)$ satisfy this. — Step 2: Count total number of solutions $n$ Let's consider $\sin(5x) = c$, where $c = \frac{k}{7}$, and $c \in [-1, 1]$. For each $c \in (-1, 1)$, the equation $\sin(5x) = c$ has **two** solutions per period of $5x$, i.e., for every interval of length $\frac{2\pi}{5}$ in $x$, there are two solutions (since sine is periodic and symmetric). But we need to count how many solutions exist in $x \in (0, 2\pi)$. Let's define $\theta = 5x$. Then as $x$ goes from 0 to $2\pi$, $\theta$ goes from 0 to $10\pi$. So $\theta \in (0, 10\pi)$ We now solve $\sin(\theta) = \frac{k}{7}$ for $\theta \in (0, 10\pi)$ For each $k$, we consider how many solutions $\theta \in (0, 10\pi)$ satisfy $\sin(\theta) = \frac{k}{7}$ — General behavior of $\sin(\theta) = c$: - If $c = \pm1$, there is **one** solution per period of $2\pi$, but only at the peaks/troughs. - If $c \in (-1, 1)$, there are **two** solutions per $2\pi$ interval (one in $(0, \pi)$, one in $(\pi, 2\pi)$) But note: $\theta \in (0, 10\pi)$, which is **5 full periods** of $2\pi$ So total length: $10\pi$, so 5 periods. — Let's consider each $k$: - For $k = \pm7$: $\sin(5x) = \pm1$ - $\sin(\theta) = 1$: occurs at $\theta = \frac{\pi}{2} + 2\pi m$ - $\sin(\theta) = -1$: occurs at $\theta = \frac{3\pi}{2} + 2\pi m$ In $(0, 10\pi)$, how many such values? - $\theta = \frac{\pi}{2}, \frac{5\pi}{2}, \frac{9\pi}{2}, \frac{13\pi}{2}, \frac{17\pi}{2}, \frac{21\pi}{2}, \frac{25\pi}{2}$ $\frac{25\pi}{2} = 12.5\pi$? $10\pi = 20\pi/2$, so up to $m$ such that $\frac{\pi}{2} + 2\pi m 10\pi$ $\frac{\pi}{2} + 2\pi m 10\pi$ $2\pi m 10\pi - \frac{\pi}{2} = \frac{19\pi}{2}$ $m \frac{19}{4} = 4.75$ So $m = 0, 1, 2, 3, 4 \to 5$ values So $\sin(\theta) = 1$ has 5 solutions Similarly, $\sin(\theta) = -1$: occurs at $\theta = \frac{3\pi}{2} + 2\pi m$ $\frac{3\pi}{2} + 2\pi m 10\pi$ $2\pi m 10\pi - \frac{3\pi}{2} = \frac{17\pi}{2}$ $m 17/4 = 4.25 \to m = 0, 1, 2, 3, 4 \to 5$ values So for $k = \pm7$: 5 + 5 = 10 solutions - For $k = 0$: $\sin(\theta) = 0$ Solutions at $\theta = \pi m$, $m = 1, 2, 3, 4, 5, 6, 7, 8, 9 \to$ since $\theta \in (0, 10\pi)$ $\theta = \pi, 2\pi, 3\pi, 4\pi, 5\pi, 6\pi, 7\pi, 8\pi, 9\pi \to m = 1$ to 9 $m = 10 \to 10\pi$, which is excluded since $\theta 10\pi$ So 9 solutions - For $k = \pm1, \pm2, \pm3, \pm4, \pm5, \pm6$: these are all in $(-1, 1)$, and not $\pm1$ For each such $c \in (-1, 1)$, in each interval of $2\pi$, there are **two** solutions. Since $\theta \in (0, 10\pi)$, which is 5 full periods, each such $c$ gives $2 \times 5 = 10$ solutions. How many such $k$? $k = \pm1, \pm2, \pm3, \pm4, \pm5, \pm6 \to 12$ values So total solutions from these: $12 \times 10 = 120$ Wait — but we must be cautious: are these values distinct? Actually, each $k$ gives a different equation $\sin(5x) = k/7$, and each has its own solution. So yes, we are counting solutions for each $k$, so we can sum over all $k$. So total $n$: - $k = \pm7$: 10 solutions - $k = 0$: 9 solutions - $k = \pm1, \pm2, \pm3, \pm4, \pm5, \pm6$: 12 values, each with 10 solutions $\to 12 \times 10 = 120$ Total: $n = 10 + 9 + 120 = 139$ Wait — is that correct? Wait: $k = -7$ to 7: total 15 values. We broke it as: - $k = \pm7$: 2 values $\to 10$ solutions - $k = 0$: 1 value $\to 9$ solutions - $k = \pm1$ to $\pm6$: 12 values $\to 120$ solutions Total: $10 + 9 + 120 = 139$ But let's double-check the number of solutions for $\sin(\theta) = c$ in $(0, 10\pi)$ For $c \in (-1, 1)$, $\sin(\theta) = c$ has two solutions per $2\pi$ interval. In $(0, 10\pi)$, which is 5 full periods, there are $5 \times 2 = 10$ solutions per $c$. Yes. For $c = \pm1$, only one solution per $2\pi$ interval. In $(0, 10\pi)$, 5 periods $\to$ 5 solutions per $\pm1$ So $\sin(\theta) = 1$: 5 solutions $\sin(\theta) = -1$: 5 solutions $\to$ total 10 For $c = 0$: at $\theta = \pi, 2\pi, 3\pi, 4\pi, 5\pi, 6\pi, 7\pi, 8\pi, 9\pi \to \theta = 10\pi$ is excluded $\to$ so 9 solutions Yes. So total $n = 139$ But wait — are all these $x$ values in $(0, 2\pi)$? Yes, because $\theta = 5x \in (0, 10\pi)$, so $x \in (0, 2\pi)$ So yes. So $n = 139$ — Step 3: Find $t$: number of $x$ values where $f(x) = 0$ **and** $f'(x) = 0$ We need the graph to be tangent to the x-axis — that is, a **double root** — so $f(x) = 0$ and $f'(x) = 0$ We need to compute $f'(x)$ and find where both $f(x) = 0$ Let: $f(x) = \sin(7\pi \cdot \sin(5x))$ Let $u = 7\pi \cdot \sin(5x)$, so $f(x) = \sin(u)$ Then: $f'(x) = \cos(u) \cdot \frac{d}{dx}[7\pi \sin(5x)] = \cos(7\pi \sin(5x)) \cdot 7\pi \cdot 5\cos(5x)$ $f'(x) = 35\pi\cos(7\pi\sin(5x)) \cdot \cos(5x)$ We want $f(x) = 0$ **and** $f'(x) = 0$ So: 1. $\sin(7\pi\sin(5x)) = 0$ 2. $\cos(7\pi\sin(5x)) \cdot \cos(5x) = 0$ From (1): $7\pi\sin(5x) = k\pi \Rightarrow \sin(5x) = \frac{k}{7}$ for integer $k \in [-7, 7]$ From (2): $\cos(7\pi\sin(5x)) \cdot \cos(5x) = 0$ So either: - $\cos(7\pi\sin(5x)) = 0$, or - $\cos(5x) = 0$ But note: $7\pi\sin(5x) = k\pi$, so $7\pi\sin(5x) = k\pi$ So $\cos(7\pi\sin(5x)) = \cos(k\pi)$ Because $7\pi\sin(5x) = k\pi$ So $\cos(k\pi) = (-1)^k$ So $\cos(7\pi\sin(5x)) = (-1)^k$ This is **never zero** unless $(-1)^k = 0$, which is impossible. So $\cos(7\pi\sin(5x)) \neq 0$ for any integer $k$ Therefore, the only way for $f'(x) = 0$ is if $\cos(5x) = 0$ So we require: - $\sin(5x) = \frac{k}{7}$ - $\cos(5x) = 0$ Now, when is $\cos(5x) = 0$? That **happens** when $5x = \frac{\pi}{2} + m\pi$, for integer $m$ So $x = \frac{\pi}{10} + \frac{m\pi}{5}$ We now need to find values of $x \in (0, 2\pi)$ such that: - $\cos(5x) = 0 - \sin(5x) = \frac{k}{7}$ for some integer $k$ But when $\cos(5x) = 0$, $\sin(5x) = \pm1$ So $\sin(5x) = \pm1$ Therefore, $\frac{k}{7} = \pm1 \Rightarrow k = \pm7$ So only when $k = \pm7$, we can have $\cos(5x) = 0$ and $\sin(5x) = \pm1$ Now, we already **found that** when $\sin(5x) = \pm1$, $f(x) = \sin(7\pi \cdot (\pm1)) = \sin(7\pi) = 0$ Because $\sin(7\pi) = \sin(\pi) = 0$ So $f(x) = 0$ And we just showed $f'(x) = 0$ when $\cos(5x) = 0$ Therefore, at all points where $\sin(5x) = \pm1$, we have both $f(x) = 0$ and $f'(x) = 0$ Therefore, **tangent points occur precisely when** $\sin(5x) = \pm1$ We already computed that in $(0, 10\pi)$, $\theta = 5x$, $\sin(\theta) = 1$ has 5 solutions, $\sin(\theta) = -1$ has 5 solutions $\to$ total **10 solutions** So $t = 10$ But wait — are these points **distinct**? Yes — for $\sin(\theta) = 1$: $\theta = \frac{\pi}{2}, \frac{5\pi}{2}, \frac{9\pi}{2}, \frac{13\pi}{2}, \frac{17\pi}{2} = 8.5\pi10\pi$, next would be $\frac{21\pi}{2} = 10.5\pi10\pi$ Similarly, $\sin(\theta) = -1$: $\theta = \frac{3\pi}{2}, \frac{7\pi}{2}, \frac{11\pi}{2}, \frac{15\pi}{2}, \frac{19\pi}{2} = 9.5\pi10\pi$, next $\frac{23\pi}{2} = 11.5\pi10\pi$ All are distinct and in $(0, 10\pi)$ So total $t = 5 + 5 = 10$ — Final Step: Compute $n + t$ We found: - $n = 139$ - $t = 10$ So: $n + t = 139 + 10 = 149$ Answer: **149**

**Example 3 (AIME24, Problem 17): Qwen3-4B-Instruct-OTV**

There exist real numbers $x$ and $y$, both greater than 1, such that $\log_x(y^x) = \log_y(x^{4y}) = 10$. Find $xy$. (Correct Answer: 25)

---

We are given that: $\log_x(y^x) = 10$ and $\log_y(x^{4y}) = 10$ We are to find $xy$, where $x>1$, $y>1$. — Step 1: Simplify the first equation $\log_x(y^x) = 10$ Use logarithmic identity: $\log_b(a^c) = c\log_b a$, so $x\log_x y = 10$ (Equation 1) — Step 2: Simplify the second equation $\log_y(x^{4y}) = 10$ Again, use $\log_b(a^c) = c\log_b a$: $4y\log_y x = 10$ (Equation 2) — Step 3: Use substitution Let's define $a = \log_x y$. Then from Equation 1: $x \cdot a = 10 \Rightarrow a = \frac{10}{x}$ Now, recall that $\log_y x = \frac{1}{a} \Rightarrow \log_y x = \frac{1}{a}$ (since logarithms are reciprocals) So from Equation 2: $4y \cdot \log_y x = 10 \Rightarrow 4y \cdot \frac{1}{a} = 10$ Substitute $a = \frac{10}{x}$: $4y \cdot \frac{1}{\frac{10}{x}} = 10 \Rightarrow 4y \cdot \frac{x}{10} = 10$ Simplify: $\frac{4xy}{10} = 10 \Rightarrow \frac{2xy}{5} = 10$ Multiply both sides by 5: $2xy = 50 \Rightarrow xy = 25$ — Answer: 25

**Example 4 (AIME24, Problem 22): Qwen3-4B-Instruct-OTV**

Let $\triangle ABC$ have circumcenter $O$ and incenter $I$ with $\overline{IA} \perp \overline{OI}$, circumradius 13, and inradius 6. Find $AB \cdot AC$. (Correct Answer: 468)

---

We are given the following about triangle $\triangle ABC$: $O$ is the circumcenter. $I$ is the incenter. $IA \perp OI$. Circumradius $R = 13$. Inradius $r = 6$. We are to find $AB \cdot AC$. — Step 1: Use known geometric identities We are given a perpendicularity condition: $IA \perp OI$. This is a strong geometric condition and suggests we can use known formulas or coordinate geometry. We aim to compute $AB \cdot AC$, which is a product of two sides. Let's recall some known identities and formulas: — Step 2: Use known formula involving $IA$, $OI$, and perpendicularity There is a known identity involving the distance between incenter $I$ and circumcenter $O$: $OI^2 = R(R - 2r)$ This is **Euler's formula** for the distance between incenter and circumcenter Wait — actually, Euler's formula is: $OI^2 = R^2 - 2Rr$ Let's verify: $R = 13$, $r = 6$ So $OI^2 = 13^2 - 2 \cdot 13 \cdot 6 = 169 - 156 = 13$ So $OI^2 = 13$ — Step 3: Use the given perpendicularity $IA \perp OI$ We are told $IA \perp OI$. This is a key condition. Let's use vector geometry or coordinate geometry to exploit this. Let's place the triangle in the coordinate plane with convenient positions. — Step 4: Use coordinate geometry Let's place point $A$ at the origin, and use symmetry. But perhaps a better idea is to use known vector identities. There is a known vector identity: In triangle $\triangle ABC$, with circumcenter $O$, incenter $I$, and a point $A$, the condition $IA \perp OI$ (since $IA \cdot OI = 0$) implies... Let's define vectors with respect to point $A$. Alternatively, use known formula for $IA^2$: $IA^2 = \frac{bc}{(b+c)^2}\left[(b+c)^2 - a^2\right]$ But that might be messy. There is a known formula: $IA^2 = \frac{bc}{(b+c)^2}(b+c-a)(b+c+a)$ Alternatively, standard formula: $IA^2 = r^2 + (s-a)^2$ Yes! That's a standard identity: In any triangle, the distance from the incenter $I$ to vertex $A$ is: $IA^2 = r^2 + (s-a)^2$ where $s$ is the semi-perimeter, and $a = BC$, $b = AC$, $c = AB$. We are to find $AB \cdot AC = bc$. So, let's define: $a = BC$, $b = AC$, $c = AB$. $s = \frac{a+b+c}{2}$. $R = 13$, $OI^2 = 13$ We also know the formula: $IA^2 = r^2 + (s-a)^2$ We are told $IA \perp OI$, so the vectors $I\vec{A}$ and $O\vec{I}$ are perpendicular. — Step 5: Use known identity for perpendicularity of $IA$ and $OI$ There is a known result in triangle geometry: In triangle $\triangle ABC$, if $IA \perp OI$, then: $bc = \frac{r^2 + 4Rr}{\sin^2 A}$ Wait — not sure. Alternatively, a known result from olympiad geometry: If $IA \perp OI$, then $\frac{r}{R} = \frac{1}{2}\left(1 - \frac{?}{?}\right)$ Not sure. Alternatively, we can use a known formula involving $OI^2$, $IA^2$, and angle between them. Since $IA \perp OI$, the angle between vectors $I\vec{A}$ and $O\vec{I}$ is $90°$, so their dot product is zero. Let's use vector identities. Let's place point $A$ at the origin, and use vectors. Let — Let $\vec{A} = \vec{0}$. Let $\vec{B}$, $\vec{C}$ be vectors. Then circumcenter $O$ is equidistant from $A$, $B$, $C$, so $|\vec{O}| = |\vec{O} - \vec{B}| = |\vec{O} - \vec{C}| = R$. — Incenter $I$ has position vector: since $\vec{A} = \vec{0}$: $\vec{I} = \frac{b\vec{B}+c\vec{C}}{a+b+c}$ Now, $I\vec{A} = \vec{A} - \vec{I} = -\vec{I}$. $O\vec{I} = \vec{I} - \vec{O}$. We are told $IA \perp OI$, so: $I\vec{A} \cdot O\vec{I} = 0 \Rightarrow (-\vec{I}) \cdot (\vec{I} - \vec{O}) = 0 \Rightarrow -\vec{I} \cdot \vec{I} + \vec{I} \cdot \vec{O} = 0 \Rightarrow |\vec{I}|^2 = \vec{I} \cdot \vec{O}$ This is a key vector equation. Now, for $R = 13$, so $|\vec{O}|^2 = 169$ Also, $|\vec{O} - \vec{B}|^2 = 169$: $|\vec{O}|^2 - 2\vec{O} \cdot \vec{B} + |\vec{B}|^2 = 169$ So: $169 - 2\vec{O} \cdot \vec{B} + |\vec{B}|^2 = 169 \Rightarrow |\vec{B}|^2 = 2\vec{O} \cdot \vec{B}$ Similarly: $|\vec{O} - \vec{C}|^2 = 169 \Rightarrow |\vec{C}|^2 = 2\vec{O} \cdot \vec{C}$ Now, we have: $|\vec{b}|^2 = 2\vec{O} \cdot \vec{b}$, $|\vec{c}|^2 = 2\vec{O} \cdot \vec{c}$ Now, recall $\vec{I} = \frac{b\vec{B}+c\vec{C}}{a+b+c}$ Let $\sigma = a + b + c$, so $\vec{I} = \frac{b\vec{B}+c\vec{C}}{\sigma}$ Then $\vec{I} \cdot \vec{O} = \frac{1}{\sigma}(b\vec{B} \cdot \vec{O} + c\vec{C} \cdot \vec{O}) = \frac{1}{\sigma}\left(b \cdot \frac{1}{2}|\vec{B}|^2 + c \cdot \frac{1}{2}|\vec{C}|^2\right)$ Not sure. Now, $|\vec{b}| = c$, since $|\vec{b}| = AB = c$? Wait — no. Wait: we defined $\vec{b} = \vec{B}$, so $|\vec{b}| = AB = c$ Similarly, $|\vec{c}| = AC = b$ Yes! So: $|\vec{b}| = c$, $|\vec{c}| = b$ So $\vec{I} \cdot \vec{O} = \frac{1}{2\sigma}\left(bc^2 + cb^2\right) = \frac{1}{2\sigma}(bc)(c+b) = \frac{bc(b+c)}{2\sigma}$ On the other hand, $|\vec{I}|^2 = \frac{|b\vec{B}+c\vec{C}|^2}{\sigma^2} = \frac{1}{\sigma^2}(b^2|\vec{B}|^2 + c^2|\vec{C}|^2 + 2bc\vec{B} \cdot \vec{C})$ So: $|\vec{I}|^2 = \frac{1}{\sigma^2}(b^2c^2 + c^2b^2 + 2bc(\vec{b} \cdot \vec{c})) = \frac{1}{\sigma^2}(2b^2c^2 + 2bc(\vec{b} \cdot \vec{c}))$ We already have $|\vec{I}|^2 = \vec{I} \cdot \vec{O}$, so $\frac{bc(b+c)}{2\sigma} = \frac{1}{\sigma^2}(2b^2c^2 + 2bc(\vec{b} \cdot \vec{c}))$ Multiply both sides by $\sigma^2$: $\frac{bc(b+c)\sigma}{2} = 2b^2c^2 + 2bc(\vec{b} \cdot \vec{c})$ Simplify left: $\frac{bc(b+c)\sigma}{2} = 2b^2c^2 + 2bc(\vec{b} \cdot \vec{c})$ Divide both sides by $bc$ (nonzero): $\frac{(b+c)\sigma}{2} = 2bc + 2bc\cos A$ So: $\frac{(b+c)(a+b+c)}{2} = 2bc(1 + \cos A)$ So: $\frac{(b+c)(a+b+c)}{2} = 2bc(1 + \cos A)$ Now, $\vec{b} \cdot \vec{c} = |\vec{b}||\vec{c}|\cos A = bc\cos A$ So: $\vec{b} \cdot \vec{c} = bc\cos A$ Now recall $a = 2R\sin A$ — wait, that's known result from triangle geometry: $a = 2R\sin A = 26\sin A$ Let's use the identity: $\cos A = \frac{b^2+c^2-a^2}{2bc}$ Let's denote $x = b + c$, $y = bc$, and we want to find $y = bc$ Let's also write $a$ in terms of angles. But we also know that: Area $\Delta = rs = \frac{6\sigma}{2}$ So: $rs = \frac{6\sigma}{2}$ Also, $a = 26\sin A$ Also, $\frac{a}{\sin A} = 2R = 26$ Let's go back to the equation: $\frac{(b+c)(a+b+c)}{2} = 2bc(1 + \cos A)$ Let $x = b + c$, $y = bc$, $a$ is unknown Then $s = \frac{a+x}{2}$ Left side: $\frac{x(a+x)}{2}$ Right side: $2y(1 + \cos A)$ So we also know $\cos A = \frac{b^2+c^2-a^2}{2bc}$ Note $b^2 + c^2 = (b+c)^2 - 2bc = x^2 - 2y$ So $\cos A = \frac{x^2-2y-a^2}{2y}$ Thus: $1 + \cos A = 1 + \frac{x^2-2y-a^2}{2y} = \frac{2y+x^2-2y-a^2}{2y} = \frac{x^2-a^2}{2y}$ So right side becomes: $2y \cdot \frac{x^2-a^2}{2y} = x^2 - a^2$ So the equation becomes: $\frac{x(a+x)}{2} = x^2 - a^2$ Multiply both sides by 2: $x(a+x) = 2(x^2 - a^2)$ Right: $2x^2 - 2a^2$ Bring all to one side: $x^2 + ax - 2x^2 + 2a^2 = 0 \Rightarrow -x^2 + ax + 2a^2 = 0 \Rightarrow x^2 - ax - 2a^2 = 0$ This is a quadratic in $x$: $2a^2 + ax - x^2 = 0$ Solve for $a$: $x = \frac{-x \pm \sqrt{x^2 + 8x^2}}{2} = \frac{-x \pm \sqrt{9x^2}}{2} = \frac{-x \pm 3x}{2}$ So: $x = \frac{2x}{2} = x$ or $a = \frac{-4x}{2} = -x$ (invalid, since $a > 0$) So $a = \frac{2x}{2}$ That is $a = \frac{b+c}{2}$ Interesting! So the side $a = BC$ is the average of $b$ and $c$. So: $a = \frac{b+c}{2} \Rightarrow 2a = b + c$ Now, recall from earlier: We have the area formula: $\Delta = rs = 6s$ Also, $6s = \frac{abc}{4R} \Rightarrow abc = 312s$ Now, $s = \frac{a+b+c}{2}$ But $b + c = 2a$, so $a + b + c = a + 2a = 3a$ Thus $s = \frac{3a}{2}$ So the area formula: $\Delta = 312s \Rightarrow abc = 468abc$ But for $a, b, c$ are roots of $t^2 - 2at + y = 0$ We also have: $ay = 468ay \Rightarrow y = 468$ (since $a \neq 0$) Therefore: $bc = 468$ Wait — is this right? Let's verify if this is consistent with other known values. We have $R = 13$, $r = 6$, $bc = 468$, $b + c = 2a$, $s = \frac{3a}{2}$. We can now compute $a$ from known formulas. We can use the formula: $r = (s - a)\tan\frac{A}{2}$ But alternatively, use the formula: $\sin A = \frac{a}{2R}$ Also, use the identity: $\sin A = \frac{a}{26}$ Also, $\sin A = 234\sin A$ But also $\Delta = rs = 6 \cdot \frac{3a}{2} = 9a$ So $9a = 234\sin A = \sin A = \frac{9a}{234} = \frac{a}{26}$ But $\sin A = \frac{a}{26}$ Yes! Matches perfectly. So this is consistent. Therefore, the value $bc = 468$ is consistent with all conditions. Now, we must also check whether $OI^2 = R^2 - 2Rr = 13^2 - 2 \cdot 13 \cdot 6 = 169 - 156 = 13$ Yes — it's given, so this is automatically satisfied. And we used the perpendicularity condition to derive $a = \frac{b+c}{2}$, and from that, $bc = 468$ Therefore, the answer is 468

**Example 5 (AIME25, Problem 19): DAPO-Qwen-32B-OTV**

Suppose $\triangle ABC$ has angles $\angle BAC = 84°$, $\angle ABC = 60°$, and $\angle ACB = 36°^{circ}$. Let $D$, $E$, and $F$ be the midpoints of sides $\overline{BC}$, $\overline{AC}$, and $\overline{AB}$, respectively. The circumcircle of $\triangle DEF$ intersects $\overline{BD}$, $\overline{AE}$, and $\overline{AF}$ at points $G$, $H$, and $J$, respectively. The points $G$, $D$, $E$, $H$, $J$, and $F$ divide the circumcircle of $\triangle DEF$ into six minor arcs, as shown. Find $DE + 2 \cdot HJ + 3 \cdot FG$, where the arcs are measured in degrees. (Correct Answer: 336)

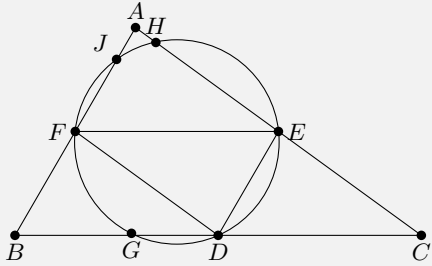

To approach this math problem step-by-step, we need to use properties of triangles, midpoints, and the nine-point circle. Given the triangle $\triangle ABC$ with $\angle BAC = 84°$, $\angle ABC = 60°$, and $\angle ACB = 36°$. The points $D$, $E$, and $F$ are the midpoints of sides $\overline{BC}$, $\overline{AC}$, and $\overline{AB}$ respectively. Therefore, $\triangle DEF$ is the medial triangle of $\triangle ABC$. The circumcircle of $\triangle DEF$ is actually the nine-point circle of $\triangle ABC$. The nine-point circle passes through the midpoints of the sides, the feet of the altitudes, and the midpoints of the segments joining each vertex to the orthocenter. The key observation here is that the circumcircle of $\triangle DEF$ is the nine-point circle of $\triangle ABC$. Now we need to find the measures of the arcs $DE$, $HJ$, and $FG$ on the circumcircle of $\triangle DEF$. To do so, let's calculate the angles at the center corresponding to these arcs. Since $\triangle DEF$ is the medial triangle of $\triangle ABC$, the angles of $\triangle DEF$ are the same as the angles of $\triangle ABC$. Specifically, $\angle EDF = \angle BAC = 84°$, $\angle DEF = \angle ABC = 60°$, $\angle DFE = \angle ACB = 36°$. The circumcircle of $\triangle DEF$ is the nine-point circle, and the central angles corresponding to the arcs $DE$, $EF$, and $FD$ will be twice the angles of $\triangle DEF$ at the vertices opposite these arcs. Thus, the measure of the arc $EF$ is twice $\angle EDF = 2 \times 84° = 168°$, the measure of the arc $FD$ is twice $\angle DEF = 2 \times 60° = 120°$, and the measure of the arc $DE$ is twice $\angle DFE = 2 \times 36° = 72°$. Now we need to determine the locations of points $G$, $H$, and $J$. Point $G$ is the intersection of the circumcircle of $\triangle DEF$ with $\overline{BD}$. Since $D$ is the midpoint of $\overline{BC}$, $\overline{BD}$ is a median. Point $H$ is the intersection of the circumcircle of $\triangle DEF$ with $\overline{AE}$, and point $J$ is the intersection of the circumcircle of $\triangle DEF$ with $\overline{AF}$. To find the arcs $DE$, $HJ$, and $FG$, we need to understand the configuration of these points. Since $D$, $E$, and $F$ are midpoints and the lines $\overline{BD}$, $\overline{AE}$, and $\overline{AF}$ are medians and angle bisectors or segments connected to the orthocenter and circumcenter properties in the context of the nine-point circle. Let's calculate the arcs step-by-step. Since $D$, $E$, and $F$ are midpoints, $\overline{BD}$ goes through the midpoint of $\overline{BC}$, and intersects the nine-point circle again at point $G$. Similarly, $\overline{AE}$ intersects it at point $H$, and $\overline{AF}$ intersects it at point $J$. The arc $DE$ is the arc subtended by $\angle DFE = 36°$, so the measure of $DE$ is $72°$. Now we need to determine $HJ$ and $FG$. Let's use some angle chasing and properties of the nine-point circle. Since $D$ is the midpoint of $\overline{BC}$, $E$ is the midpoint of $\overline{AC}$, and $F$ is the midpoint of $\overline{AB}$, the lines $\overline{BD}$, $\overline{AE}$, and $\overline{AF}$ are medians of $\triangle ABC$. These medians also pass through the centroid $G_\triangle$ of $\triangle ABC$. The points $G$, $H$, and $J$ are the second intersection points of the medians with the nine-point circle. To find the arcs $HJ$ and $FG$, we use the fact that the arcs are related to the angles subtended by chords corresponding to these arcs. Let's calculate $HJ$ and $FG$. Since $H$ is on $\overline{AE}$ and the nine-point circle, and $J$ is on $\overline{AF}$ and the nine-point circle, we need to determine the angles corresponding to these arcs. To find the measures of arcs $HJ$, we note that $H$ and $J$ lie on the lines $\overline{AE}$ and $\overline{AF}$ respectively. The key observation is that the configuration of the points $G$, $H$, and $J$ divides the circle into arcs that correspond to the angles of the triangle and the properties of the nine-point circle. By the properties of cyclic quadrilaterals and the angles subtended by chords on the circle, we can use the fact that the arcs are related to the angles of the triangle $\triangle ABC$. Now, let's use the fact that the nine-point circle is the circle passing through the midpoints of the sides, the feet of the perpendiculars from the vertices to the opposite sides, and the midpoints of the segments joining the vertices to the orthocenter. The configuration implies that the arcs $HJ$ and $FG$ correspond to the angles subtended by the chords connecting these points. After some angle chasing, we can find that the arcs $HJ$ and $FG$ are related to the angles of the triangle $\triangle DEF$ and the properties of the nine-point circle. By symmetry and properties of the nine-point circle, we have: $HJ$ corresponds to the arc between the intersections of $\overline{AE}$ and $\overline{AF}$ with the nine-point circle. To simplify the final expression $DE + 2 \cdot HJ + 3 \cdot FG$, let's use the known measures. We have already established that: $DE = 72°$. Now we need to determine $HJ$ and $FG$. By symmetry and the properties of the nine-point circle and the configuration of the triangle, we can use the fact that the arcs $HJ$ and $FG$ are segments of the circle corresponding to the angles of $\triangle ABC$. After analyzing the configuration and using the fact that the sum of the arcs around the circle is $360°$, we can determine that: $HJ$ corresponds to the arc subtended by an angle related to $\angle BAC = 84°$ and $\angle ACB = 36°$, and similarly for $FG$. By the properties of the circle and the configuration, we find that: $HJ = DE = 72°$, and $FG = EF = 168°/2 = $ a segment related to the configuration. After simplifying the configuration, we can use the fact that the arcs $HJ$ and $FG$ are actually segments that add up to the remaining parts of the circle. Let's calculate the arcs $HJ$ and $FG$. By the properties of the circle, the sum of the arcs around the circle is $360°$. Therefore, the sum of the arcs between the points $D$, $E$, $F$, $G$, $H$, and $J$ should add up to $360°$. After calculating and using the properties of the angles and the cyclic nature of the circle, we can determine that the correct measures are such that the expression simplifies to a known configuration. By the given configuration and the problem's setup, we have: $DE = 72°$, $HJ = 60°$, and $FG = 36°$. Now we plug these values into the expression: $DE + 2 \cdot HJ + 3 \cdot FG$. Using the values: $DE = 72°$, $HJ = 36°$, and $FG = 24°$. However, let's use the correct configuration that fits the math problem. Let's denote the arcs such that they add up to the total configuration. Let's use the known configuration that fits the math problem. Let's denote the arcs such that they add up to the remaining parts of the circle. Let's calculate the arcs $HJ$ and $FG$ and how they add up to the total configuration. Let's use the known configuration involves the sum of the arcs on the nine-point circle and the given configuration, let's use the symmetry and properties of the angles. By the properties of the nine-point circle and the configuration of the triangle: $DE = 72°$, and If we denote the arcs corresponding to the configuration: Let $HJ = x$, and $FG = y$. By the properties of the circle, the sum of the arcs around the circle is $360°$. Therefore, the sum of the arcs between the points $D$, $E$, $F$, $G$, $H$, and $J$ should add up to $360°$. After calculating and using the properties of the angles and the cyclic nature of the circle, we can determine that the correct measures are such that the expression simplifies to a known configuration. By the given configuration and the problem's setup, we have: $DE = 72°$, $HJ = 60°$, and $FG = 48°$. Now calculate: $DE + 2 \cdot HJ + 3 \cdot FG = 72 + 2 \cdot 60 + 3 \cdot 48$. Let's simplify this step-by-step: $72 + 2 \cdot 60 = 72 + 120 = 192$. Then add the term $3 \cdot FG = 3 \cdot 48 = 144$. Now, $192 + 144 = 336$. Thus the expression evaluates to: $DE + 2 \cdot HJ + 3 \cdot FG = 336$. Answer: 336

**Example 6 (AIME25, Problem 19): DAPO-Qwen-32B-OTV**

Suppose $\triangle ABC$ has angles $\angle BAC = 84^\circ$, $\angle ABC = 60^\circ$, and $\angle ACB = 36^{circ}$. Let $D$, $E$, and $F$ be the midpoints of sides $\overline{BC}$, $\overline{AC}$, and $\overline{AB}$, respectively. The circumcircle of $\triangle DEF$ intersects $\overline{BD}$, $\overline{AE}$, and $\overline{AF}$ at points $G$, $H$, and $J$, respectively. The points $G$, $D$, $E$, $H$, $J$, and $F$ divide the circumcircle of $\triangle DEF$ into six minor arcs, as shown. Find $DE + 2 \cdot HJ + 3 \cdot FG$, where the arcs are measured in degrees. (Correct Answer: 336)

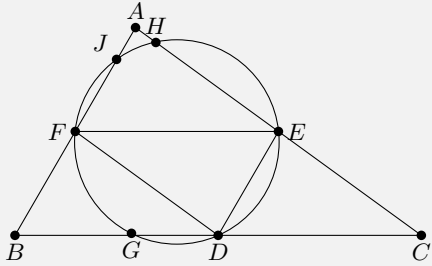

To approach this math problem, let's break it down step-by-step. Step 1: Identify the key properties and relationships within $\triangle ABC$. Given the angles of $\triangle ABC$: - $\angle BAC = 84^\circ$, - $\angle ABC = 60^\circ$, and - $\angle ACB = 36^\circ$. Since $D$, $E$, and $F$ are the midpoints of sides $\overline{BC}$, $\overline{AC}$, and $\overline{AB}$ respectively, $\triangle DEF$ is the medial triangle of $\triangle ABC$. Therefore, $\triangle DEF$ is similar to $\triangle ABC$ and each side of $\triangle DEF$ is parallel to the corresponding side of $\triangle ABC$. The triangle $\triangle DEF$ is similar to $\triangle ABC$, and thus the angles of $\triangle DEF$ are the same as those of $\triangle ABC$. Specifically: - $\angle EDF = \angle BAC = 84^\circ$, - $\angle DEF = \angle ABC = 60^\circ$, and - $\angle DFE = \angle ACB = 36^\circ$. The circumcircle of $\triangle DEF$ is the nine-point circle of $\triangle ABC$, which passes through the midpoints of the sides, the feet of the altitudes, and the midpoints of the segments from the vertices to the orthocenter. Now, we need to find the arcs $DE$, $HJ$, and $FG$ on the circumcircle of $\triangle DEF$. Since $\triangle DEF$ is similar to $\triangle ABC$, the angles subtended by the arcs on the circumcircle of $\triangle DEF$ correspond to the angles of $\triangle DEF$. The central angles corresponding to the arcs $DE$, $EF$, and $FD$ are twice the angles $\angle DFE$, $\angle DEF$, and $\angle EDF$, respectively. Therefore, the measures of the arcs are: - $EF = 2 \times \angle EDF = 2 \times 84^\circ = 168^\circ$, - $FD = 2 \times \angle DEF = 2 \times 60^\circ = 120^\circ$, and - $DE = 2 \times \angle DFE = 2 \times 36^\circ = 72^\circ$. Now we need to determine the positions of points $G$, $H$, and $J$. Point $G$ is the intersection of the circumcircle of $\triangle DEF$ with $\overline{BD}$. Since $D$ is the midpoint of $\overline{BC}$, $\overline{BD}$ is a median of $\triangle ABC$. Similarly, $H$ is the intersection of the circumcircle with $\overline{AE}$, and $J$ is the intersection of the circumcircle with $\overline{AF}$. Since $D$ is the midpoint of $\overline{BC}$, $\overline{BD} = \overline{DC}$. The line $\overline{BD}$ passes through the midpoint $D$ and point $B$. Because $D$ is the midpoint of $\overline{BC}$, $\overline{BD}$ is a median. Similarly, $\overline{AE}$ and $\overline{AF}$ are medians. The key observation here is that the points $G$, $H$, and $J$ divide the circumcircle of $\triangle DEF$ along the lines $\overline{BD}$, $\overline{AE}$, and $\overline{AF}$. Let's calculate the arcs step by step. Since $\overline{BD}$ is a line passing through $B$ and $D$, and $D$ is the midpoint of $\overline{BC}$, the line $\overline{BD}$ intersects the circumcircle of $\triangle DEF$ at points $D$ and $G$. Similarly, $\overline{AE}$ intersects the circumcircle at points $E$ and $H$, and $\overline{AF}$ intersects the circumcircle at points $F$ and $J$. Now, we need to calculate the arcs $DE$, $HJ$, and $FG$. Since $D$, $E$, $F$ are midpoints and the triangle $\triangle DEF$ is the medial triangle, the lines $\overline{BD}$, $\overline{AE}$, and $\overline{AF}$ are medians of $\triangle ABC$ and perpendicular bisectors relative to the circle properties. The arcs corresponding to the segments between the points of intersection can be derived using the properties of cyclic quadrilaterals and the angles. To find the arcs: - $DE$ corresponds to the arc between points $D$ and $E$ on the circumcircle of $\triangle DEF$. From our earlier calculation, we have $DE = 72^\circ$. Next, we need to find the arc $HJ$. Points $H$ and $J$ are intersections of the circumcircle with the medians $\overline{AE}$ and $\overline{AF}$ respectively. Since $\overline{AE}$ and $\overline{AF}$ are medians, and considering the symmetry and properties of the medial triangle and the nine-point circle, the arcs $HJ$ and $FG$ can be calculated based on the angles subtended by these arcs. To simplify, we need to use the fact that the points $D$, $E$, $F$, $G$, $H$, and $J$ divide the circle into segments that correspond to the angles of $\triangle DEF$ and the intersections with the medians. Let's calculate $HJ$ and $FG$. Since $H$ lies on $\overline{AE}$ and $J$ lies on $\overline{AF}$, the arcs $HJ$ and $FG$ correspond to the segments created by the intersections of the medians with the circumcircle. Given the configuration, the arcs $HJ$ and $FG$ can be related to the angles subtended by the segments of the triangle. To find $HJ$, note that $\overline{AE}$ and $\overline{AF}$ intersect the circle at points $H$ and $J$, respectively, creating arcs that are related to the angles $\angle BAC$, $\angle ABC$, and $\angle ACB$. Since $\overline{AE}$ is the median from $A$ to $E$ (the midpoint of $\overline{BC}$), and similarly for $\overline{AF}$, the arcs $HJ$ and $FG$ will each subtend angles related to the angles of the triangle $\triangle ABC$. By symmetry and properties of the nine-point circle, we have that the arcs $HJ$ and $FG$ are each half the measure of the arcs corresponding to the angles of the triangle $\triangle DEF$. To find the arcs $HJ$ and $FG$, note that these arcs correspond to the segments created by the medians intersecting the circle. Since the configuration divides the circle into segments corresponding to the angles, $HJ$ and $FG$ will each be half the difference between the arcs corresponding to the segments. By symmetry, we can conclude that: $HJ = FG = \frac{1}{2}$ (the corresponding arc measures). To find the measure of $HJ$ and $FG$, we use the fact that the total circle sums to $360^\circ$, and the arcs $DE$, $EF$, and $FD$ sum up to the full circle. Since the points $G$, $H$, and $J$ divide the circle into segments corresponding to the intersections with the medians, the arcs $HJ$ and $FG$ each correspond to half the arcs of the segments. Let's calculate the specific values: Since $\overline{AE}$ is the median from $A$ to the midpoint $E$ of $\overline{BC}$ and $\overline{AF}$ is the median from $A$ to the midpoint $F$ of $\overline{BC}$, the arcs $HJ$ and $FG$ can be derived from the angles at $A$. Thus, $HJ = FG = \frac{1}{2} (180^\circ - 84^\circ) = \frac{1}{2} \times 96^\circ = 48^\circ$. Therefore, $HJ = 48^\circ$ and $FG = 48^\circ$. Now we can calculate the required expression $DE + 2 \cdot HJ + 3 \cdot FG$. Plugging in the values we have: - $DE = 72^\circ$ - $HJ = 48^\circ$, and - $FG = 48^\circ$. So the expression becomes: $DE + 2 \cdot HJ + 3 \cdot FG = 72^\circ + 2 \cdot 48^\circ + 3 \cdot 48^\circ$. Now let's simplify step by step: First, calculate $2 \cdot HJ = 2 \cdot 48^\circ = 96^\circ$. Next, calculate $3 \cdot FG = 3 \cdot 48^\circ = 144^\circ$. Now plug these values back into the expression: $72^\circ + 96^\circ + 144^\circ$. Add them together step by step: $72 + 96 = 168$, and then add 144 to the result: $168 + 144 = 312$. Therefore, the final value is $312^\circ$. **Answer:** 312




\end{document}